\documentclass[journal]{IEEEtai}

\usepackage[colorlinks,urlcolor=blue,linkcolor=blue,citecolor=blue]{hyperref}

\usepackage{color,array}

\usepackage{microtype}
\usepackage{graphicx}
\usepackage{booktabs}

\usepackage{multicol}
\usepackage{multirow}
\usepackage{url}
\usepackage{soul}
\usepackage{caption}
\usepackage{subcaption}

\usepackage{amsmath}
\usepackage{mathtools}
\usepackage{amsthm}
\usepackage{amsfonts}
\usepackage{amssymb}
\usepackage{algorithm}
\usepackage{algorithmic}

\usepackage[capitalize,noabbrev]{cleveref}

\theoremstyle{plain}
\newtheorem{theorem}{Theorem}[section]

\theoremstyle{definition}
\newtheorem{definition}[theorem]{Definition}

\theoremstyle{remark}

\definecolor{label2}{RGB}{255, 230, 149}
\definecolor{duong}{RGB}{0, 0, 0}
\definecolor{khoi}{RGB}{0, 0, 0}

\newcommand{\vizsizelambda}{0.24}
\newcommand{\vizsizeB}{0.23}
\newcommand{\vizsizeweight}{0.45}

\begin{document}

\title{Revisiting LARS for Large Batch Training Generalization of Neural Networks}

\author{Khoi Do$^*$, Duong Nguyen$^*$, Hoa Nguyen, Long Tran-Thanh, Nguyen-Hoang Tran, Quoc-Viet Pham

\thanks{This work has been submitted to the IEEE for possible publication. Copyright may be transferred without notice, after which this version may no longer be accessible.}
\thanks{$*$ means equal contribution}
\thanks{Khoi Do and Hoa Nguyen were with the School of Electrical and Electronics Engineering, Hanoi University of Science and Technology.}
\thanks{Duong Nguyen was with Pusan National University.}
\thanks{Long Tran-Thanh was with the University of Warwick.}
\thanks{Nguyen-Hoang Tran was with the University of Sydney.}
\thanks{Quoc-Viet Pham was with Trinity College of Dublin.}
}

\markboth{Journal of IEEE Transactions on Artificial Intelligence, Vol. 00, No. 0, Month 2024}
{First A. Author \MakeLowercase{\textit{et al.}}: Bare Demo of IEEEtai.cls for IEEE Journals of IEEE Transactions on Artificial Intelligence}

\maketitle

\begin{abstract}
This paper investigates Large Batch Training techniques using layer-wise adaptive scaling ratio (LARS) across diverse settings. In particular, we first show that a state-of-the-art technique, called LARS with warm-up, tends to be trapped in sharp minimizers early on due to redundant ratio scaling. Additionally, a fixed steep decline in the latter phase restricts deep neural networks from effectively navigating early-phase sharp minimizers. 
To address these issues, we propose Time Varying LARS (TVLARS), a novel algorithm that replaces warm-up with a configurable sigmoid-like function for robust training in the initial phase. TVLARS promotes gradient exploration early on, surpassing sharp optimizers and gradually transitioning to LARS for robustness in later phases. Extensive experiments demonstrate that TVLARS consistently outperforms LARS and LAMB in most cases, with up to 2\% improvement in classification scenarios. Notably, in all self-supervised learning cases, TVLARS achieves up to 10\% performance improvement.
\end{abstract}

\begin{IEEEImpStatement}
The rapid growth in deep learning, especially in building foundation models (e.g. Large (Visual) Language Model, etc) has led to increasing demand for efficient and scalable training techniques, particularly with large batch sizes. Large batch training enhances both speed and hardware utilization for building AI models. However, large batch training occurs in unstable performance. Current methods like LARS with warm-up, while effective, often face challenges in maintaining performance, especially in the presence of sharp minimizers. This paper introduces Time Varying LARS (TVLARS), a novel approach that addresses these limitations by enabling more robust training and improved generalization across diverse settings. With demonstrated improvements of up to 2\% in classification tasks and up to 10\% in self-supervised learning scenarios, TVLARS has the potential to significantly enhance the efficiency and effectiveness of large-scale, accurate, and reliable deep learning applications. 
\end{IEEEImpStatement}


\section{Introduction}
Large Batch Training (LBT) plays a crucial role in modern Deep Learning (DL), offering efficiency gains through parallel processing, improved generalization with exposure to diverse samples, and optimized memory and hardware utilization \cite{survey, opt-method}. These advantages make LBT particularly suitable for training large Deep Neural Network (DNN) models and Self-Supervised Learning (SSL) tasks\cite{zbontar2021barlow,chen2020big,chen2020simple}, in which finding good representation latent is crucial. Nonetheless, the present application of LBT with conventional gradient-based methods often necessitates the use of heuristic tactics and results in compromised generalization accuracy \cite{2017-DL-TrainingLonger,2018-LR-SharpMinia}.

Numerous methods \cite{2021-DL-CLARS,2020-DL-LAMB,2020-DL-LAMBC} have been explored to address the performance issues associated with LBT. Among these methods, Layer-wise Adaptive Rate Scaling (LARS) \cite{2017-DL-LARS} and its variation, LAMB \cite{2020-DL-LAMB}, has gained significant popularity. Fundamentally, LARS employs adaptive rate scaling to improve gradient descent on a per-layer basis. As a result, training stability is enhanced across the layers of the DNN model. Despite the benefits, LARS faces instability in the initial stages of the LBT process, leading to slow convergence, especially with large batches (LB). Implementing a warm-up strategy is effective in reducing the LARS adaptive rate and stabilizing the learning process for larger batch sizes. Nevertheless, achieving stable convergence with LARS necessitates a specific duration of time. Moreover, as the batch size increases significantly, the performance of LARS experiences a notable decline \cite{2017-DL-LARS,2021-DL-CLARS}.

In light of these challenges, \textcolor{khoi}{we raise two considerable questions:
\textbf{1)} \textit{What are the fundamental factors precipitating instability in the early training stages of LARS and warm-up?} \textbf{2)} \textit{Is there an alternative methodology that is proficient to LBT by tackling the underlying issues?}} 
This paper aims to provide answers to these questions in the following way:
First, we conduct numerous empirical experiments to gain a more comprehensive understanding of the principles and to grasp the limitations inherent in the series of LARS techniques. These experiments yield intriguing observations, particularly during the initial phase of LARS training. \textcolor{duong}{Concerning the LARS rationales}, the layer-wise learning rate $\gamma^k_t$ was determined to be high due to the Layer-wise Normalization Rate (LNR), i.e., $\Vert w \Vert / \Vert \nabla \mathcal{L}\Vert$. This effect was attributed to the near-zero value of the Layer-wise Gradient Norm (LGN). This infinitesimal value of LGN (e.g., $\Vert \nabla \mathcal{L}\Vert \ll 1$) was a consequence of getting trapped into sharp minimizers, characterized by large positive eigenvalue of Hessian during the initial phase \cite{2018-LR-SharpMinia, 2017-ShartMinima-Generalize}. As a result, this phenomenon leads to an \emph{explosion of the scaled gradient}. 

For the incorporation of warm-up into LARS, \emph{it takes considerable unnecessary steps to scale the gradient to a threshold that enables escape from the initial sharp minimizers} (refers to Figure~\ref{fig:lars-tvlars}). Furthermore, because of the fixed decay in the Learning Rate (LR), the warm-up process \emph{does not effectively encourage gradient exploration over the initial sharp minimizers} and struggles to adapt to diverse datasets. 

Given these findings, we answer the second research question by \emph{proposing a new algorithm called Time-Varying LARS} (TVLARS), which enables gradient exploration for LARS in the initial phase while retaining the stability of other LARS family members in the latter phase. Instead of using warm-up (which suffers from major aforementioned issues), TVLARS overcomes sharp minimizes by taking full advantage of a high initial LR (i.e., target LR) and inverted sigmoid function to enhance training stability, aligning with theories about sharp minimizers in LBT \cite{2018-LR-SharpMinia}. Our contributions can be summarized as follows: 
\begin{itemize}
    \item We provide new insights on the behavior of two canonical LBT techniques, namely LARS and LAMB, via comprehensive empirical evaluations, to understand how they enhance the performance of LBT.
    \item We identify crucial causes of the well-known issue of performance drops of warm-up in many cases, and argue that these shortcomings may arise from the lack of understanding of sharp minimizers in LBT.
    \item We propose a simple alternative technique, named TVLARS, which is more aligned with the theories about sharp minimizers in LBT and can avoid the potential issues of the warm-up approach. 
    \item We conduct various experimental evaluations, comparing its performance against other popular baselines. Our experiments' results demonstrate that TVLARS significantly outperforms the state-of-the-art benchmarks, such as WA-LARS under the same delay step and target LR.
\end{itemize}

\section{Related Works}
\textbf{Large-batch training.} In \cite{codreanu2017scale}, several LR schedulers are proposed to figure out problems in LBT, especially the Polynomial Decay technique which helps ResNet50 converge within 28 minutes by decreasing the LR to its original value over several training iterations. Since schedulers are proven to be useful in LBT, \cite{2017-DL-LARS}, \cite{peng2018megdet} suggested an LR scheduler based on the accumulative steps and a GPU cross Batch Normalization. \cite{gotmare2018closer}, besides, investigates deeper into the behavior of cosine annealing and warm-up strategy then shows that the latent knowledge shared by the teacher in knowledge distillation is primarily disbursed in the deeper layers. Inheriting the previous research, \cite{2018-DL-LinearScaling} proposed a hyperparameter-free linear scaling rule used for LR adjustment by constructing a relationship between LR and batch size as a function.

\textbf{Adaptive optimizer.} Another orientation is optimization improvement, starting by \cite{2017-DL-LARS}, which proposed LARS optimizers that adaptively adjust the LR for each layer based on the local region. To improve LARS performance, \cite{jia2018highly} proposed two training strategies including low-precision computation and mixed-precision training. In contrast, \cite{pmlr-v162-liu22n} authors propose JointSpar and JointSpar-LARS to reduce the computation and communication costs. On the other hand, Accelerated SGD \cite{yamazaki2019accelerated}, is proposed for training DNN in large-scale scenarios. In \cite{2020-DL-LAMB}, \cite{slamb}, new optimizers called LAMB and SLAMB were proved to be successful in training Attention Mechanisms along with the convergence analysis of LAMB and LARS. With the same objective AGVM \cite{xue2022largebatch} is proposed to boost RCNN training efficiency. Authors in \cite{2020-DL-LAMBC}, otherwise proposed a variant of LAMB called LAMBC which employs trust ratio clipping to stabilize its magnitude and prevent extreme values. CLARS \cite{2021-DL-CLARS}, otherwise is suggested to exchange the traditional warm-up strategy owing to its unknown theory. 

\section{Backgrounds}
\textbf{Notation.} We denote by $w_t \in \mathbb{R}^d$ the model parameters at time step $t$. For any \textcolor{khoi}{empirical loss function} $\ell: \mathbb{R}^d \rightarrow \mathbb{R}$, $\nabla \ell(x,y\vert w_t)$ denotes the gradient with respect to $w_t$. We use $\Vert\cdot\Vert$ and $\Vert\cdot\Vert_1$ to denote $l_2$-norm and $l_1$-norm of a vector, respectively. We start our discussion by formally stating the problem setup. In this paper, we study a non-convex stochastic optimization problem of the form:
\begin{align}
    \underset{w\in \mathbb{R}^d}{\min} \mathcal{L}(w) \triangleq \mathbb{E}_{(x,y)\sim P(\mathcal{X}, \mathcal{Y})} [\ell(x,y\vert w)] + \frac{\lambda}{2}\Vert w \Vert^2
\end{align}
where $(x, y)$ and $P(\mathcal{X}, \mathcal{Y})$ represent data sample and its distribution, respectively.

\textbf{LARS.} To deal with LBT, the authors in \cite{2017-DL-LARS} proposed LARS. Suppose a neural network has $K$ layers, we have $w=\{w^1, w^2,\ldots ,w^K\}$. The LR at the layer $k$ is updated as follows:
\begin{align}
    \gamma^k_t = \gamma_\textrm{scale}\times\eta\times\frac{\Vert w^k_t \Vert}{\Vert \frac{1}{\mathcal{B}}\sum_{i = 0}^{\mathcal{B} - 1}\nabla \ell(x_i, y_i \vert w^k_t) + w_d\Vert},
\label{eq:LARS}
\end{align}
where $\gamma_\textrm{scale} = \gamma_\textrm{tuning}\times\frac{\mathcal{B}}{\mathcal{B}_\textrm{base}}$ is the base LR \cite{2017-DL-LARS}, $\eta$ is the LARS coefficient for the LBT algorithm, \textcolor{khoi}{$\mathcal{B}$ and $\mathcal{B}_{\rm base}$ denote the batch size used in practice and the base batch size that yields optimal performance (which usually be the small batch size), respectively \cite{2022-LR-Decay1, 2014-DL-WeirdTrick}}, and $w_d$ is weight decay to ensure non-zero division, for simplicity, it is not included in the dominator later on. We denote $\Vert \frac{1}{B}\sum_{i = 0}^{\mathcal{B} - 1}\nabla \ell(x_i, y_i | w^k_t) \|$ as the LGN and $\|w_t^k\|$ as LWN. For simplicity, we denote the LGN as $\Vert\nabla\mathcal{L}(w_t^k) \Vert$. The LNR is defined as $\| w^k_t \|/\| \frac{1}{B}\sum_{i = 0}^{\mathcal{B} - 1}\nabla \ell(x_i, y_i | w^k_t) \|$, which has the objective of normalizing the learning rate at each layer $k$. Despite its practical effectiveness, there is inadequate theoretical insight into LARS. Additionally, without implementing warm-up techniques \cite{2018-DL-LinearScaling}, LARS tends to exhibit slow convergence or divergence during initial training.

\section{Experimental Study}\label{sec:lars-mys}
This section explores how the LARS optimizers contribute to LBT. By revealing the mechanism of LBT, we provide some insights for improving LARS performance. To understand the detrimental impact of lacking warm-up procedure in current state-of-the-art LBT techniques, we employ the LARS and LAMB optimizers on vanilla classification problems to observe the convergence behavior. We conduct empirical experiments on CIFAR10 \cite{2010-DL-Cifar10}. 

\subsection{On the principle of LARS}\label{sec:lars-principle}
First, we revisit the LARS algorithm, which proposes adaptive rate scaling. Essentially, LARS provides a set of learning rates that adapts \textcolor{khoi}{individually} to each layer of the DNN, as shown in Equation~(\ref{eq:LARS}). From a geometric perspective on layer $k$ of the DNN, LWN $\Vert w^k_t \Vert$ can be seen as the magnitude of the vector containing all components in the Euclidean vector space. Similarly, the LGN can be regarded as the magnitude of the gradient vector of all components in the vector space. Thus, the LNR can be interpreted as the number of distinct pulses in Hartley's law \cite{1928-DL-HartleyLaw}.

By considering the LNR, we can adjust the layer-wise gradient based on the LWN. In other words, instead of taking the \textcolor{duong}{vanilla model update step $\nabla \textcolor{khoi}{\mathcal{L}(w^k_t)}$ at every layer $k$, we perform an update step as a percentage of the LWN. The proportional gradient update can be expressed as}
\begin{equation}
\Vert w^k_t \Vert\times\frac{\nabla\mathcal{L}(w^{k,j}_t)}{\Vert \nabla \mathcal{L}(w^k_t)\Vert},
\end{equation}
$\nabla\mathcal{L}(w^{k,j}_t)/\Vert \nabla \mathcal{L}(w^k_t) \Vert$ represents the estimation of the percentage of gradient magnitude on each parameter $j$ with respect to the LGN of layer $k$. It becomes apparent that the layer-wise learning rate of LARS only influences the percentage update to the layer-wise model parameters. However, it does not address the issue of mitigating the problem of sharp minimizers in the initial phase.
\begin{figure}
\centering
\begin{subfigure}{.47\columnwidth}
  \centering
  \includegraphics[width=\linewidth]{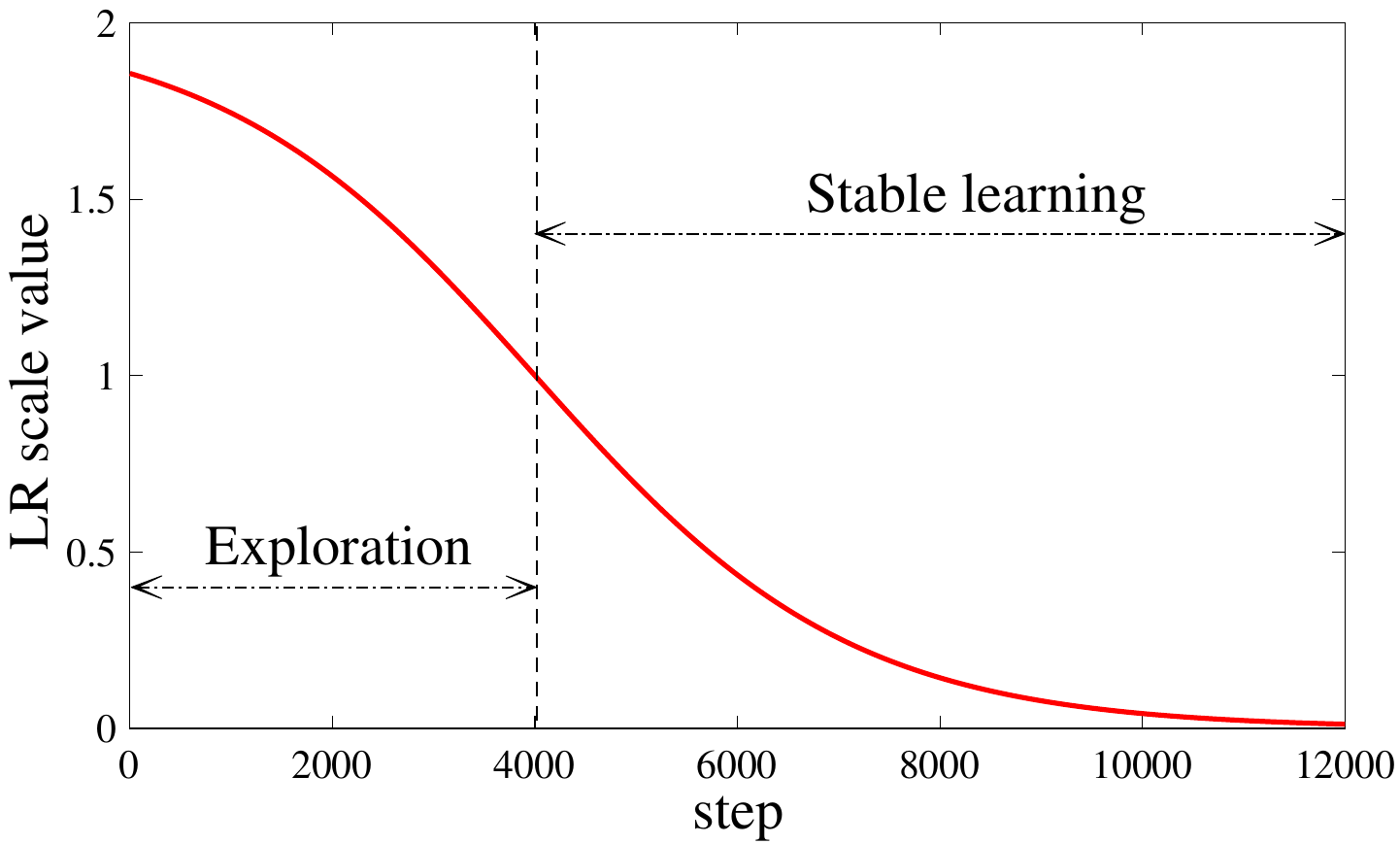}
  \caption{Non-WA-LARS.}
  \label{fig:lars-tvlars-nwa}
\end{subfigure}%
\begin{subfigure}{.52\columnwidth}
  \centering
  \includegraphics[width=\linewidth]{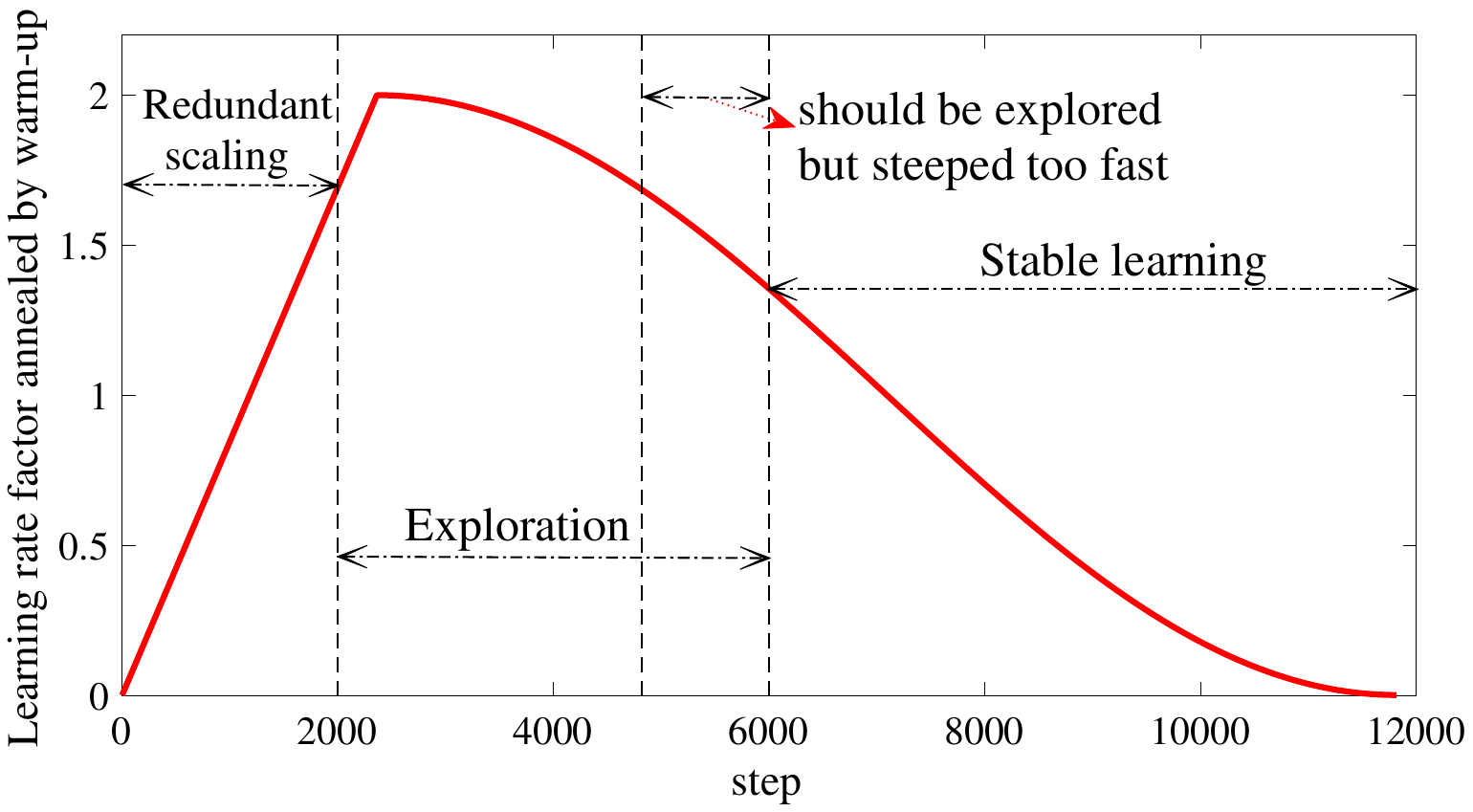}
  \caption{WA-LARS.}
  \label{fig:lars-tvlars-wa}
\end{subfigure}
\caption{Scaling the learning rate in two different strategies.}
\label{fig:lars-tvlars}
\end{figure}

\subsection{LARS and the importance of \textcolor{khoi}{the} warm-up}\label{sec:lars-discussion}
The warm-up strategy is considered an important factor of LBT, which enhances the model performance and training stability \cite{2020-DL-LAMB}, \cite{gotmare2018closer}, \cite{2018-DL-LinearScaling}, \cite{2017-DL-LARS}. The warm-up strategy in Eq.~(\ref{eq:wa-tech}) linearly scales the learning rate from $0$ to the target one ($t \leq d_{\rm wa}$), where $d_{\rm wa}$ is the number of warm-up steps. When $t > d_{\rm wa}$, $\gamma_t^k$ started to be annealed by a cosine function (see Figure \ref{fig:lars-tvlars-wa}). However, warm-up theoretically increases the learning rate from a very small value causing \emph{redundant training time} and making model \emph{easy to be trapped} into sharp minimizers. Therefore, we analyze the vitality of warming up as well as its potential issues. 
\begin{equation}\label{eq:wa-tech}
    \gamma_t^k = \begin{cases}
        \gamma_{\rm scale} \times \frac{t}{d_{\rm wa}},~\textrm{when } t \leq d_{\rm wa} \\
        \frac{1}{2}\left[1 + \cos(\frac{t - d_{\rm wa}}{T - d_{\rm wa}})\right],~\textrm{when } t > d_{\rm wa}
    \end{cases}
\end{equation}
\textbf{Quantitative results.} Our quantitative results on CIFAR10 reveal a decline in accuracy performance when contrastingly runs with and without a warm-up strategy. In particular, the LARS without a warm-up technique exhibits greater training instability, characterized by fluctuating accuracy. Moreover, the performance decline becomes more significant, especially with larger batch sizes (refers to Figure \ref{fig:lars-test-16k}).

\begin{figure}
    \centering
    \includegraphics[width=\linewidth]{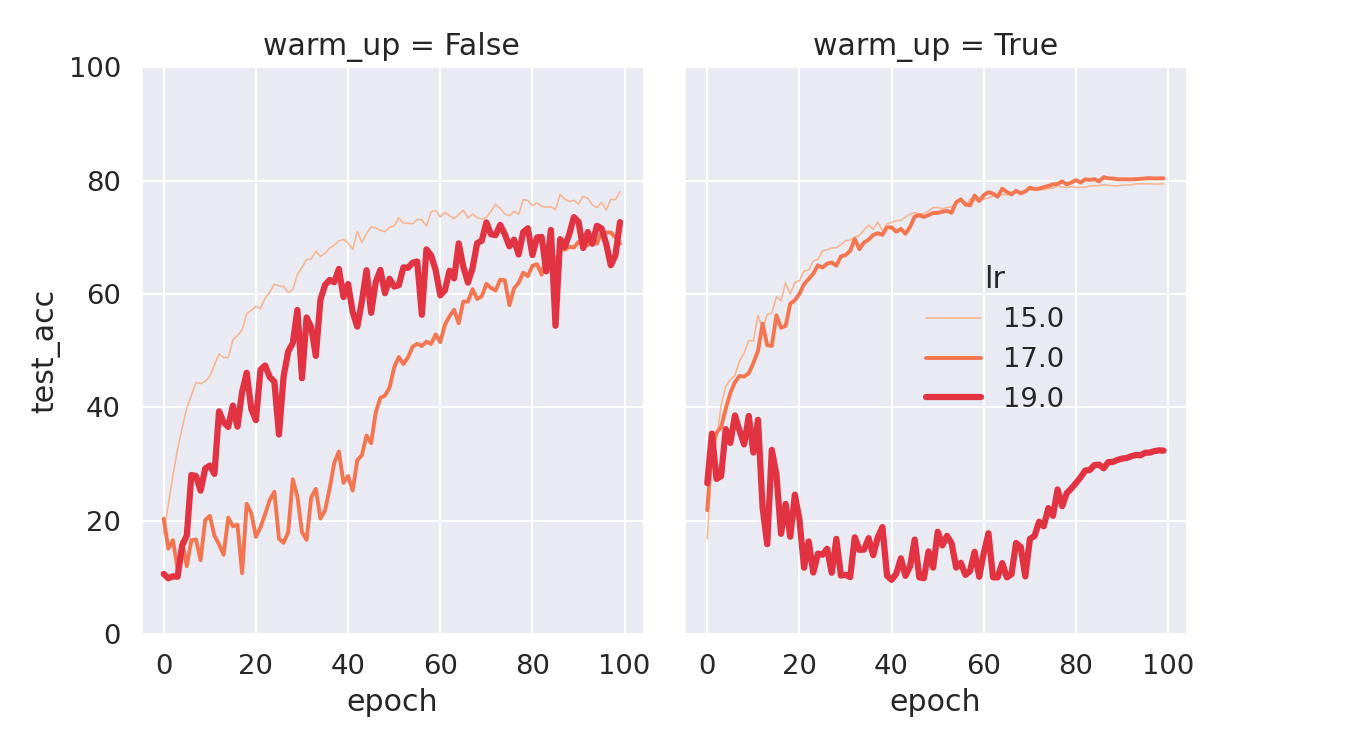}
    \caption{Comparison between LARS trained with and without warm-up.}
    \label{fig:lars-test-16k}
\end{figure}

\textcolor{black}{\textbf{From adaptive ratio to LBT performance.} To enhance our comprehension of the adaptive rate scaling series, we conduct thorough experiments analyzing LNR in LARS. Each result in our study includes two crucial elements: the test loss during model training and the corresponding LNR.}
\begin{figure*}[!ht]
     \centering
     \begin{subfigure}[b]{0.24\textwidth}
         \centering
         \includegraphics[width=\textwidth]{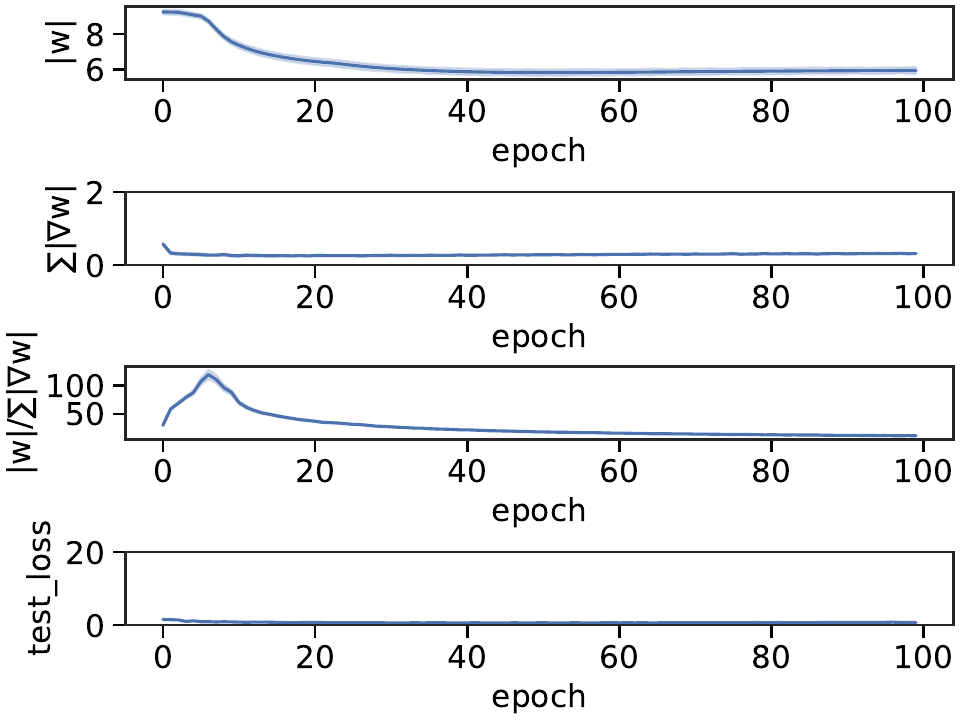}
         \caption{NOWA-LARS $1024$}
         \label{fig:bigwhy-16K-lr15-wowa}
     \end{subfigure}
     \begin{subfigure}[b]{0.24\textwidth}
         \centering
         \includegraphics[width=\textwidth]{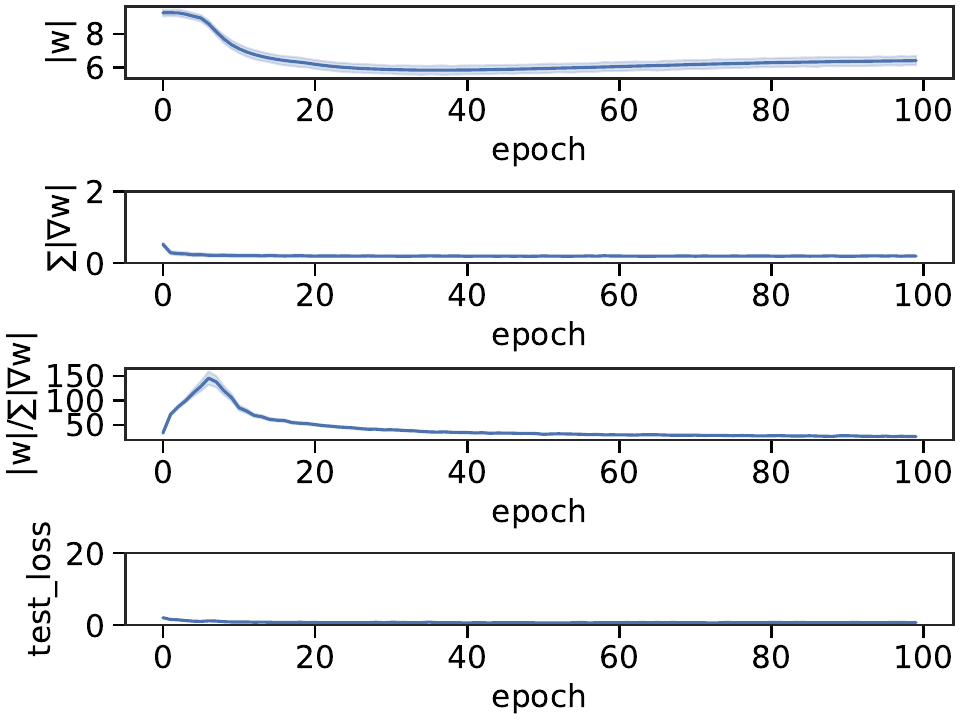}
         \caption{NOWA-LARS $2048$}
         \label{fig:bigwhy-16K-lr17-wowa}
     \end{subfigure}
     \begin{subfigure}[b]{0.24\textwidth}
         \centering
         \includegraphics[width=\textwidth]{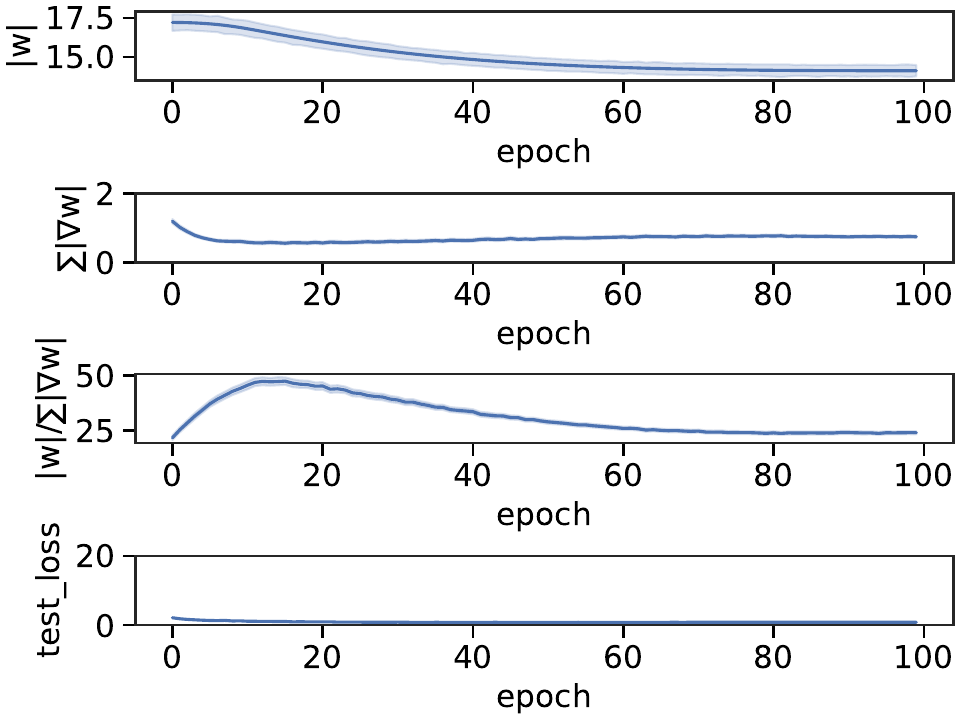}
         \caption{WA-LARS $1024$}
         \label{fig:bigwhy-8K-lr12-wa}
     \end{subfigure}
     \begin{subfigure}[b]{0.24\textwidth}
         \centering
         \includegraphics[width=\textwidth]{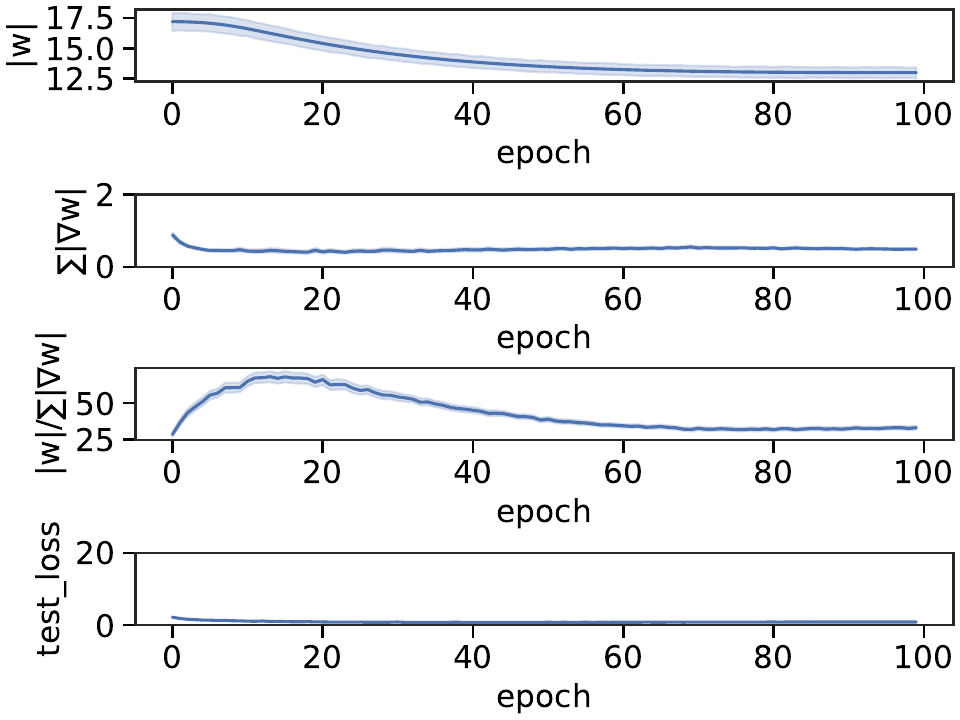}
         \caption{WA-LARS $2048$}
         \label{fig:bigwhy-8K-lr15-wa}
     \end{subfigure}
     \begin{subfigure}[b]{0.24\textwidth}
         \centering
         \includegraphics[width=\textwidth]{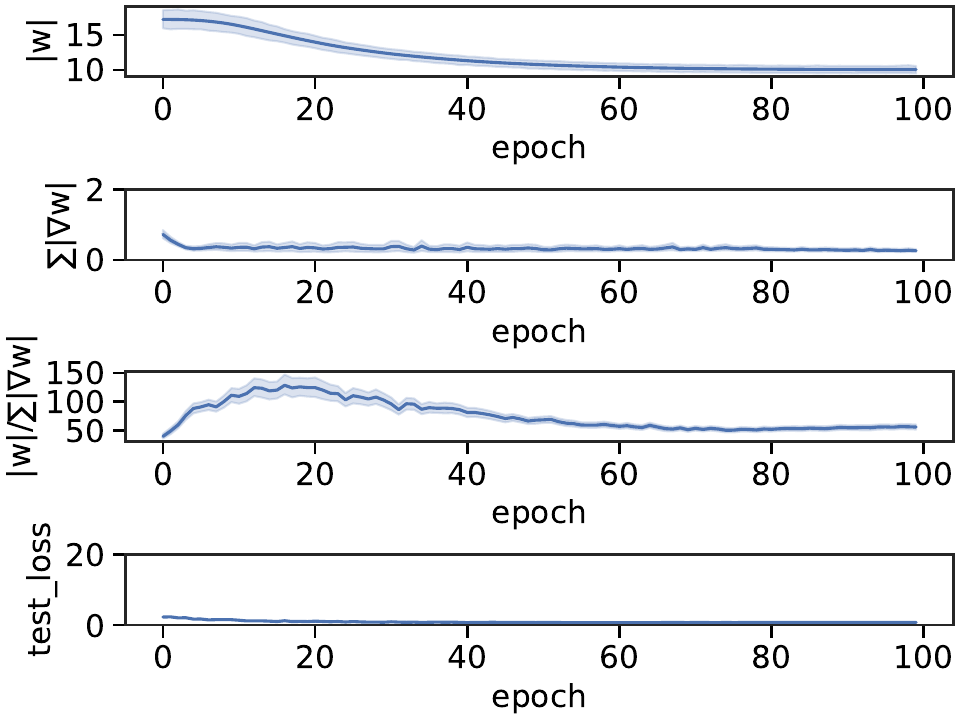}
         \caption{NOWA-LARS $8192$}
         \label{fig:bigwhy-16K-lr19-wa}
     \end{subfigure}
     \begin{subfigure}[b]{0.24\textwidth}
         \centering
         \includegraphics[width=\textwidth]{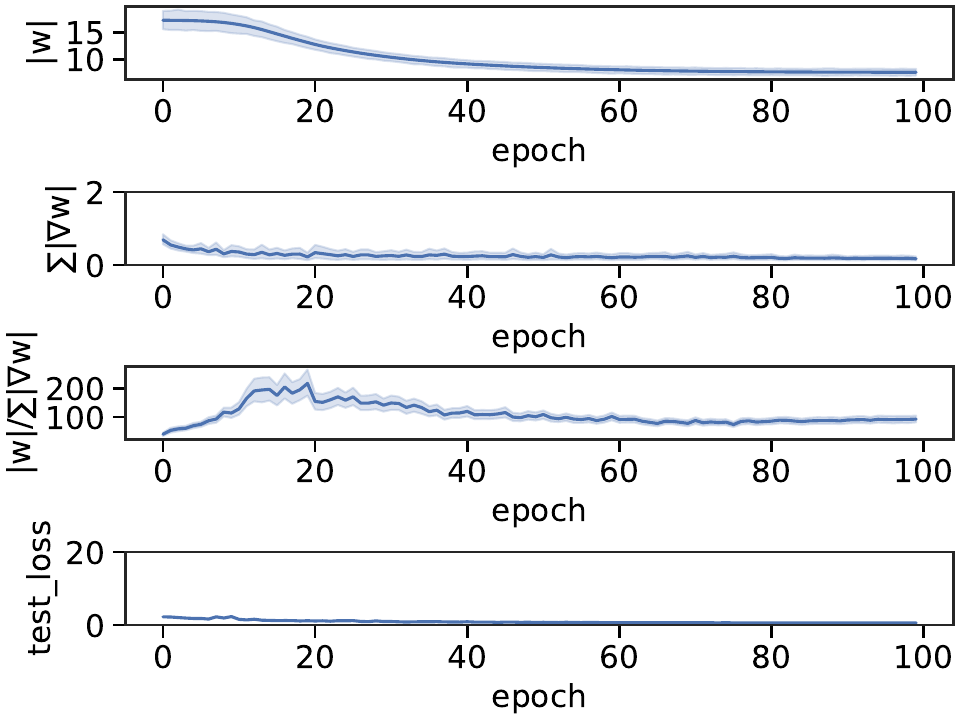}
         \caption{NOWA-LARS $16384$}
         \label{fig:bigwhy-16K-lr19-wowa}
     \end{subfigure}
     \begin{subfigure}[b]{0.24\textwidth}
         \centering
         \includegraphics[width=\textwidth]{image-lib/adv_loss_ratio_plot_result/cifar10_resnet18_xavier_normal/lars/adv_ratio_8192_0dot1_lars-warm.pdf}
         \caption{WA-LARS $8192$}
         \label{fig:bigwhy-16K-lr15-wa}
     \end{subfigure}
     \begin{subfigure}[b]{0.24\textwidth}
         \centering
         \includegraphics[width=\textwidth]{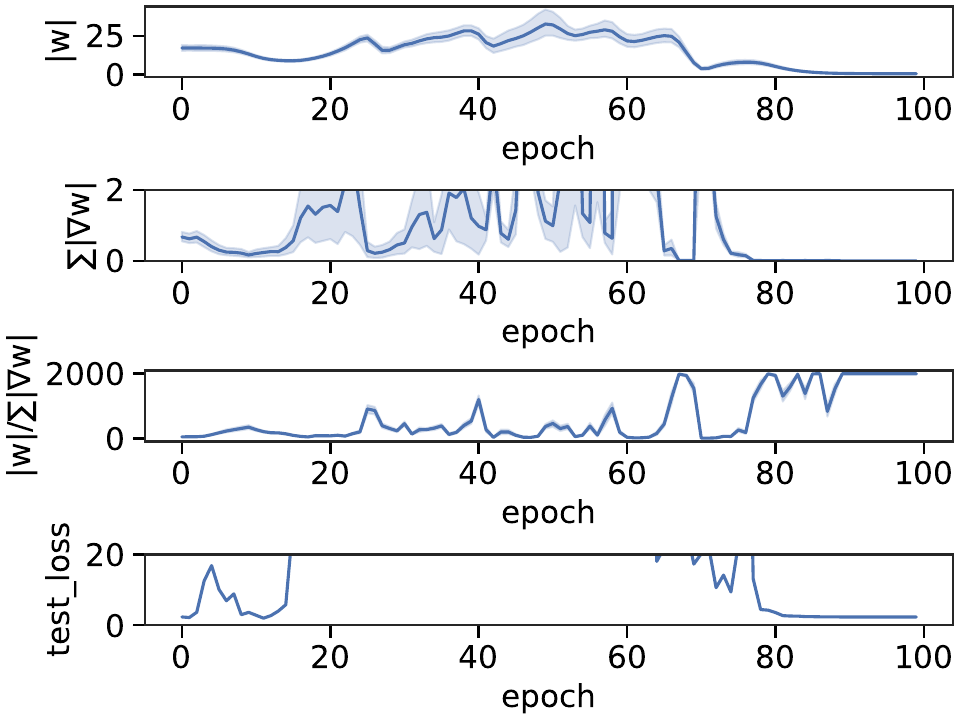}
         \caption{WA-LARS $16384$}
         \label{fig:bigwhy-16K-lr17-wa}
     \end{subfigure}
     \caption{This figure illustrated the quantitative performance of LARS ($B = 16\rm{K})$ conducted with a warm-up and without a warm-up strategy (NOWA-LARS). Each figure contains 4 subfigures, which indicate the LWN $\Vert w \Vert$, LGN $\Vert\nabla \mathcal{L}(w)\Vert$, and LNR $\Vert w \Vert/ \Vert\nabla \mathcal{L}(w)\Vert$ of all layers, and test loss value in the y axis.}
     \label{fig:bigwhy1}
\end{figure*}

In our study, we examined the detailed results of WA-LARS (Figure \ref{fig:bigwhy1}). Based on these findings, we make observations regarding the convergence and the behavior of the LNR:

\textbf{1)} During the initial phase of the successfully trained models (characterized by a significant reduction in test loss), the LNR tends to be high, indicating a higher learning rate. 

\textbf{2)} High variance in the LNR indicates significant exploration during training, resulting in noticeable fluctuations in LWM. Conversely, when training requires stability, the LNR variance decreases.

\textbf{3)} It is necessary to impose an upper threshold on the LNR to prevent divergence caused by values exceeding the range of $\Vert \gamma_t^k \times \nabla\mathcal{L} \Vert$ over the LWN $\Vert w_t \Vert$ \cite{2017-DL-LARS}. The absence of the warm-up technique in LARS often leads to the LNR surpassing this upper threshold. This issue is addressed by the WA-LARS. Specifically, compared with non warm-up LARS (NOWA-LARS) at batch sizes, where the corresponding LNR during the initial phase is limited to values such as 0.15, 0.2 (refer to Figures~\ref{fig:bigwhy-16K-lr15-wowa},\ref{fig:bigwhy-16K-lr17-wowa}).

\textbf{4)} The LNR in the WA-LARS is regulated by a more gradual incline shown in Figure \ref{fig:bigwhy1}.  On the other hand, in contrast to the WA-LARS, the NOWA-LARS exhibits a steep decline in the LNR. 
\textbf{Extensive Study.} To gain deeper insights, we conduct additional experiments, and the results are presented in Figure~\ref{fig:bigwhy1}. Through the analysis of Section~\ref{sec:lars-discussion} and the insights derived from Figure~\ref{fig:bigwhy1}, the reduction in the LNR can be attributed to the rapid decrease in the LWN over time. This phenomenon occurs in tandem with the exponential reduction of the LGN. Consequently, we can deduce that the superior performance of the learning model about larger batch sizes is a consequence of the gradual decrease in the LWN.

From Figure~\ref{fig:bigwhy1}, we can see that there is always a gradual and consistent decrease of $\Vert w^k_t \Vert$ whenever the models have stable convergence. \textcolor{khoi}{Hence, we can conceptualize the model parameters as residing on a hypersphere, with the use of gradient descent technique to explore the topological space of this hypersphere} (refer to Figure~\ref{fig:lars-hypersphere}). This exploration commences from the hypersphere's edge, indicated by $\Vert w^k_t \Vert = w_{\textrm{max}}$, and progresses toward its center, characterized by $\Vert w^k_t \Vert = 0$. Nevertheless, in cases of exploding gradient issues, significant fluctuations in $\Vert w^k_t \Vert$ can disrupt the functionality of this hypothesis.

In the context of NOWA-LARS, the LNR experiences a steeper decline due to the rapid reduction in the LWN $\Vert w^k_t \Vert$. \textcolor{khoi}{The notable reduction in rate can be interpreted as a rapid exploration of the parameter vector hypersphere. This swift exploration may overlook numerous potential searches, risking a failure to identify global minimizers. In contrast, with WA-LARS, the search across the parameter space is more gradual, ensuring a more stable and comprehensive exploration of the parameter hypersphere (Figure \ref{fig:bigwhy1})}

Based on the aforementioned observations, we conclude that the primary challenges encountered in the context of LARS stem from two key issues: \textbf{1)} a \emph{high LNR} and \textbf{2)} \emph{substantial variance in the LWN}. These dual challenges pose significant obstacles to LARS's effective performance. As a solution, the warm-up process aims to prevent the occurrence of exploding gradients during the initial phase by initially setting the learning rate coefficient to a significantly low value and gradually increasing it thereafter. Nevertheless, as the batch size reaches exceedingly large values, the performance of WA-LARS appears to deteriorate, as the behavior of LNR and LWN deviates from the previously mentioned observations. We believe that the application of the warm-up technique may be somewhat lacking in a comprehensive understanding. Consequently, we are motivated to delve deeper into the characteristics of sharp minimizers within the LBT, seeking a more profound insight into LARS and the warm-up process.

\subsection{Shortcomings of the warm-up} \label{sec:understand-empirical-results}
\textbf{The degradation in learning performance when the batch size becomes large.} As mentioned in Section~\ref{sec:lars-principle}, the LARS technique only influences the percentage update to their layer-wise model parameters to stabilize the gradient update behavior. However, the learning efficiency is not affected by the LARS technique. To gain a better understanding of LARS performance as the batch size increases significantly, our primary goal is to establish an upper limit for the unbiased gradient, which is similar to \cite{2020-FL-FedNova} (i.e., the variance of the batch gradient). We first adopt the following definition: 
\begin{definition}
    A gradient descent $g^t_i$ at time $t$ using reference data point $x_i$ is a composition of a general gradient $\Bar{g}^t$ and a variance gradient $\Delta g^t_i$. For instance, we have $ g^t_i = \Bar{g}^t + \Delta g^t_i,$ where the variance gradient $\Delta g^t_i$ represents the perturbation of gradient descent over the dataset. The general gradient represents the invariant characteristics over all perturbations of the dataset. 
\label{def:single-gradient}
\end{definition}
This definition leads to the following theorem that shows the relationship between unbiased gradient \cite{2020-FL-FedNova} and the batch size. 
\begin{theorem}[Unbiased Large Batch Gradient]
    Given $\Bar{g}^t$ as mentioned in Definition~\ref{def:single-gradient}, $g^t_\mathcal{B}$ is the batch gradient with batch size $\mathcal{B}$. Given $\sigma^2$ is the variance for point-wise unbiased gradient as mentioned in \cite{2020-FL-FedNova}, we have the stochastic gradient with $\mathcal{B}$ is an unbiased estimator of the general gradient and has bounded variance: $\mathbb{E}_{(x, y) \sim P(\mathcal{X}, \mathcal{Y})} \left[ \Bar{g}^t - g^t_\mathcal{B} \right] \leq {\sigma^2}/{\mathcal{B}}$, 
\label{theorem:batch-unbiased-grad}
\end{theorem}

\begin{proof}
    Revisit the Definition~\ref{def:single-gradient}, we have: 
    \begin{align}
        g^t_i = \Bar{g}^t + \Delta g^t_i.
    \end{align}
    In applying the LB gradient descent with batch size $\mathcal{B}$, we have: 
    \begin{align}
        g^t_\mathcal{B} = \frac{1}{\mathcal{B}} \sum^{\mathcal{B}}_{i=1} g^t_i = \frac{1}{\mathcal{B}} \sum^{\mathcal{B}}_{i=1} \Bar{g}^t + \Delta g^t_i = \Bar{g}^t +\frac{1}{\mathcal{B}} \sum^{\mathcal{B}}_{i=1} \Delta g^t_i.
    \label{eq:lbl-gd}
    \end{align}
    Apply the $L^2$ Weak Law Theorem 2.2.3 in \cite{2010-MF-Probability}, we have: $g^t_\mathcal{B} \leq \Bar{g}^t + \frac{\sigma^2}{\mathcal{B}}$, which can be also understood as: 
    \begin{align}
        \mathbb{E}_{x_i, y_i \sim P(\mathcal{X}, \mathcal{Y})} \left[ \Bar{g}^t - g^t_\mathcal{B} \right] \leq {\sigma^2}/{\mathcal{B}}
    \end{align}
\end{proof}

Theorem~\ref{theorem:batch-unbiased-grad} demonstrates that utilizing an LB size $\mathcal{B}$ during training results in more stable gradients. However, there are two significant concerns with this which come with negative implications. Firstly, the stability of the gradient is influenced by the LB. Consequently, in scenarios where the LNR experiences rapid reduction (discussed in NOWA-LARS in Section~\ref{sec:lars-discussion}), the \emph{gradient descent process is potentially trapped into sharp minimizers during the initial stages} \cite{2018-LR-SharpMinia}. Secondly, due to the steep decline in the LNR, the exploration across the hypersphere of $w_t^k$ occurs excessively swiftly (specifically, from $w_\textrm{max}$ to $0$). LB techniques lack the exploratory characteristics that are available in small batch (SB) methods and often focus excessively on narrowing down to the sharp minimizer that is closest to the starting point \cite{2018-LR-SharpMinia}.
\begin{figure}
\centering
\includegraphics[width=\linewidth]{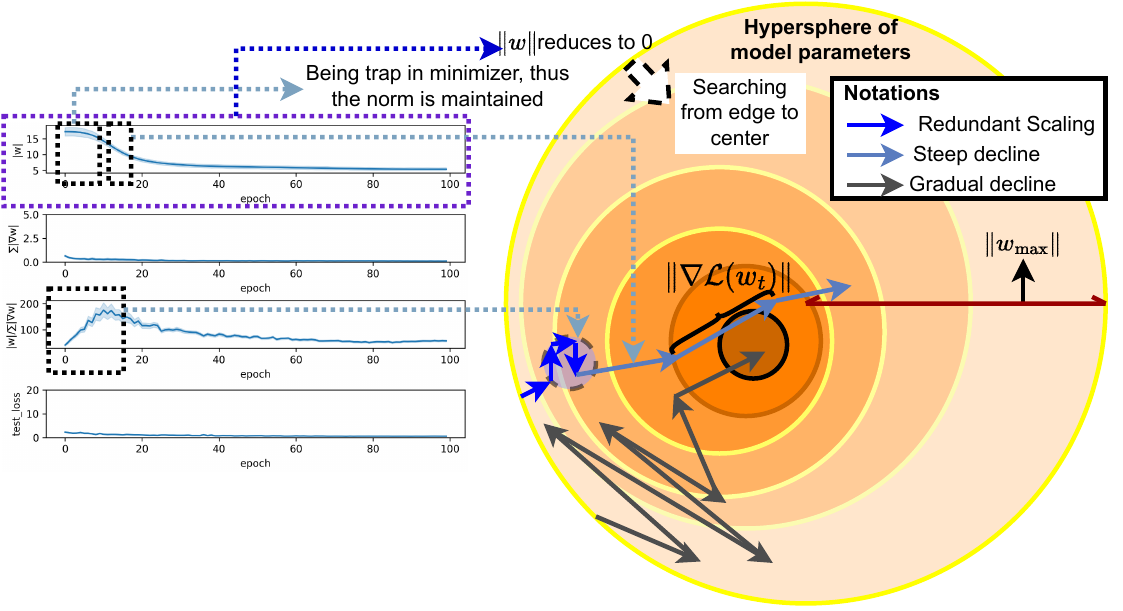}
\caption{Illustration of gradient descent behavior from the perspective of model parameter hypersphere.}
\label{fig:lars-hypersphere}
\end{figure}
\textbf{Redundant Ratio Scaling in Warm-up LARS.} Warm-up \cite{gotmare2018closer} involves initially scaling the base learning rate and subsequently reducing it to facilitate gradient exploration. However, our findings indicate that gradually increasing the base learning rate from an extremely low value before gradient exploration is unnecessary (Figure~\ref{fig:lars-tvlars}). Consequently, when we multiply the base learning rate with the LNR (which tends to be low in initial rounds), $\gamma^k_t$ will be extremely low accordingly. 

When $\gamma^k_t$ is too small, particularly at the initial stage, \emph{learning fails to avoid memorizing noisy data} \cite{2019-DL-HowLRDecay}. Moreover, when the model gets trapped in the sharp minimizers during the warm-up process, due to the steepness of the sharp minimizers, the model will be \emph{unable to escape from the sharp minima}. Furthermore, apart from the high variance of the gradient of mini-batch training, the gradient of the LBT is stable as mentioned in Theorem~\ref{theorem:batch-unbiased-grad}. Therefore, the LBT is halted until the learning rate surpasses a certain threshold (see Figure~\ref{fig:lars-hypersphere} for an illustration).

\begin{table*}[!ht]
\centering
\caption{Accuracy (\%) of LARS, LAMB, and TVLARS. Weight initialization is Xavier Uniform}
\label{tab:perform_cls}
\setlength\aboverulesep{0.5ex}
\resizebox{0.9\linewidth}{!}{%
\begin{tabular}{lccccccccccccc}
\toprule
\multicolumn{2}{l}{\textbf{Problem}}  & \multicolumn{6}{c}{\textbf{Classification} ($\lambda = 10^{-4}$)}                                   & \multicolumn{6}{c}{\textbf{SSL - Barlow Twins} ($\lambda = 10^{-5}$)}                                                                                           \\ \midrule
\multicolumn{2}{l}{\textbf{Data set}} & \multicolumn{3}{c}{\textbf{CIFAR10}}             & \multicolumn{3}{c}{\textbf{ImageNet}}       & \multicolumn{3}{c}{\textbf{CIFAR10}}                                                     & \multicolumn{3}{c}{\textbf{ImageNet}}                             \\ \midrule
\multicolumn{2}{l}{Learning rate}     & 1              & 2              & 3              & 1              & 2              & 3              & 1              & 2                                  & 3                                  & 1              & 2                         & 3                         \\ \midrule
LARS   & \multirow{3}{*}{$B = 512$}   & 74.64          & 77.49          & 79.64          & 33.72          & 36.60          & 38.92          & 51.31          & \multicolumn{1}{l}{58.89}          & \multicolumn{1}{l}{60.76}          & 18.32          & 19.52                     & 21.12                     \\
LAMB   &                              & 57.88          & 64.76          & 78.39          & 20.28          & 41.40          & 37.52          & 10.01          & \multicolumn{1}{l}{12.01}          & \multicolumn{1}{l}{67.31}          & 13.46          & \multicolumn{1}{l}{17.75} & \multicolumn{1}{l}{27.37} \\
TVLARS &                              & \textbf{78.92} & \textbf{81.42} & \textbf{81.56} & \textbf{37.56} & \textbf{39.28} & \textbf{39.64} & \textbf{69.96} & \multicolumn{1}{l}{\textbf{69.72}} & \multicolumn{1}{l}{\textbf{70.54}} & \textbf{29.02} & \textbf{31.01}            & \textbf{31.44}            \\ \midrule
\multicolumn{2}{l}{Learning rate}     & 2              & 3              & 4              & 2              & 3              & 4              & 2              & 3                                  & 4                                  & 2              & 3                         & 4                         \\ \midrule
LARS   & \multirow{3}{*}{$B = 1024$}  & 74.70          & 78.83          & 80.52          & 33.8           & 36.84          & 38.52          & 61.13          & \multicolumn{1}{l}{67.03}          & 68.98                              & 19.96          & \multicolumn{1}{l}{19.36} & \multicolumn{1}{l}{20.38} \\
LAMB   &                              & 52.06          & 57.83          & 79.98          & 16.88          & 31.84          & 37.76          & 12.03          & \multicolumn{1}{l}{15.03}          & \multicolumn{1}{l}{69.49}          & 16.44          & \multicolumn{1}{l}{16.70} & \multicolumn{1}{l}{24.53} \\
TVLARS &                              & \textbf{81.84} & \textbf{82.58} & \textbf{82.54} & \textbf{39.60} & \textbf{39.16} & \textbf{39.40} & \textbf{67.38} & \multicolumn{1}{l}{\textbf{69.80}} & \multicolumn{1}{l}{\textbf{71.13}} & \textbf{28.98} & \textbf{27.46}            & \textbf{28.32}            \\ \midrule
\multicolumn{2}{l}{Learning rate}     & 5              & 6              & 7              & 5              & 6              & 7              & 5              & 6                                  & 7                                  & 5              & 6                         & 7                         \\ \midrule
LARS   & \multirow{3}{*}{$B = 2048$}  & 75.22          & 79.49          & 81.1           & 34.48          & 38.04          & 40.44          & 52.42          & 57.35                              & \multicolumn{1}{l}{52.42}          & 19.92          & 20.33                     & 19.98                     \\
LAMB   &                              & 43.12          & 47.88          & 81.43          & 12.60          & 18.56          & 39.92          & 10.01          & \multicolumn{1}{l}{15.64}          & \multicolumn{1}{l}{60.08}          & 18.67          & 21.98                     & 23.43                     \\
TVLARS &                              & \textbf{81.52} & \textbf{82.22} & \textbf{82.44} & \textbf{39.56} & \textbf{39.56} & \textbf{41.68} & \textbf{62.61} & \textbf{61.71}                     & \textbf{61.05}                     & \textbf{26.15} & \textbf{23.43}            & \textbf{27.08}            \\ \midrule
\multicolumn{2}{l}{Learning rate}     & 8              & 9              & 10             & 8              & 9              & 10             & 8             & 9                                 & 10                                 & 8             & 9                        & 10                        \\ \midrule
LARS   & \multirow{3}{*}{$B = 4096$}  & 75.63          & 80.96          & 82.49          & 34.24          & 38.56          & 40.40          & 52.77          & \multicolumn{1}{l}{55.74}          & \multicolumn{1}{l}{55.98}          & 19.45          & 20.20                     & 20.25                     \\
LAMB   &                              & 22.32          & 37.52          & 81.53          & 05.24          & 10.76          & 39.00          & 48.93          & 50.19                              & 53.38                              & 13.78          & 15.66                     & 21.45                     \\
TVLARS &                              & \textbf{80.9}  & \textbf{80.96} & \textbf{81.16} & \textbf{38.28} & \textbf{39.96} & \textbf{41.64} & \textbf{57.96} & \textbf{58.28}                     & \textbf{60.46}                     & \textbf{24.29} & \textbf{24.29}            & \textbf{24.92}            \\ \midrule
\multicolumn{2}{l}{Learning rate}     & 10             & 12             & 15             & 10             & 12             & 15             & 10              & 12                                  & 15                                 & 10              & 12                         & 15                        \\ \midrule
LARS   & \multirow{3}{*}{$B = 8192$}  & 77.59          & 81.75          & 82.5           & 34.36          & 38.12          & 42.00          & 50.29          & \multicolumn{1}{l}{52.78}          & \multicolumn{1}{l}{09.65}          & 19.69          & 20.79                     & 21.62                     \\
LAMB   &                              & 16.14          & 19.85          & 81.78          & 01.12          & 02.48          & 39.20          & 42.22          & 45.26                              & 52.39                              & 19.60          & 23.44                     & 23.55                     \\
TVLARS &                              & \textbf{81.16} & \textbf{82.42} & \textbf{82.74} & \textbf{36.48} & \textbf{39.20} & 42.32          & \textbf{54.14} & \textbf{53.56}                     & \textbf{52.24}              & \textbf{23.41} & \textbf{23.47}            & \textbf{24.26}            \\ \midrule
\multicolumn{2}{l}{Learning rate}     & 15             & 17             & 19             & 15             & 17             & 19             & 15             & 17                                 & 19                                 & 15             & 17                        & 19                        \\ \midrule
LARS   & \multirow{3}{*}{$B = 16384$} & 79.46          & 80.62          & 38.57          & 35.60          & 39.72          & 41.72          & 48.27          & 48.26                              & 49.05                              & 20.97          & 21.54                     & 22.42                     \\
LAMB   &                              & 14.82          & 11.97          & 77.16          & 00.84          & 00.92          & 37.28          & 42.26          & 44.02                              & 49.65                              & 20.34          & 23.58                     & 23.78                     \\
TVLARS &                              & \textbf{80.20} & 80.42          & \textbf{80.42} & \textbf{38.56} & \textbf{40.32} & \textbf{42.08} & \textbf{49.16} & \textbf{49.84}                     & \textbf{50.15}                     & \textbf{25.82} & \textbf{24.95}            & \textbf{25.91}            \\ \bottomrule
\end{tabular}
}
\end{table*}

\section{Methodology} \label{sec:methodology}
After conducting and comprehending the experiential quantification in Section \ref{sec:lars-mys}, we have identified the issues of warm-up LARS. It becomes evident that the $\gamma^k_t$ is not well-aligned with the characteristics of sharp minimizer distributions in the context of LB settings. Specifically, during the initial phases of the search process, the loss landscape tends to exhibit numerous sharp minimizers\cite{2017-ShartMinima-Generalize, 2017-DL-LARS, 2021-DL-CLARS} that necessitate sufficiently high gradients to facilitate efficient exploration\cite{2018-LR-SharpMinia, 2017-ShartMinima-Generalize}. Furthermore, the learning rate must be adjustable so that the LBT can be fine-tuned to match the behavior of different datasets and learning models. Our method instead directly uses the high initial learning rate to escape sharp minimas\cite{2018-LR-SharpMinia}.  

To this end, we propose a novel algorithm TVLARS (see Algorithm~\ref{alg:TVLARS}) where the main contribution of TVLARS is highlighted) for LBT that aims to take full advantage of the strength of LARS and warm-up strategy along with drawbacks avoidance. A key idea of TVLARS is to ensure the following characteristics: 1) elimination of incremental phase of base learning rate to eliminate redundant unlearnable processes, 2) a configurable base LR function that can be tuned for different data and model types, and 3) a lower threshold for stability and inheriting LARS robustness.

\textbf{1) Initiating Exploration Excitation.} Although warm-up strategy enhances model training stability (Section~\ref{sec:lars-mys}), learning from a strictly small LR prevents the model from tackling poor sharp minimizers, appearing much near initialization point \cite{2022-LR-Decay1}, \cite{2018-LR-SharpMinia}, \cite{2017-ShartMinima-Generalize}. Otherwise, as a result of the steep decline in adaptive LNR, the exploration around the hypersphere of $w^k_t$ is restricted (Theorem \ref{theorem:batch-unbiased-grad}) and does not address the sharp minimizer concern (Section \ref{sec:lars-principle}). Moreover, warm-up does not fulfill the need for LBT training because of the unnecessary linear scaling (Section \ref{sec:understand-empirical-results}). We construct TVLARS as an optimizer that uses a high initial learning rate owing to its ability of early sharp minima evasion, which enhances the loss landscape exploration.  
\begin{align}\label{eq:time-vary}
    \phi_(t) = \frac{1}{\alpha + e^{\psi_t}} + \gamma_{\min} \quad \textrm{where} \quad \psi_t = \lambda(t - d_{\rm e})
\end{align}
\begin{algorithm}[!h]
    \centering
    \caption{
    TVLARS algorithm}\label{alg:TVLARS}
    \small
    \begin{algorithmic}[1]
    \STATE {\bfseries Require:} $w_t^k\in \mathbb{R}^d$, LR $\{\gamma_t^k\}^T_t$, delay factor $\lambda$, batch size $\mathcal{B}$, delay epoch $d_{\rm e}$, scaling factor $\alpha$, time-varying factor $\phi_(t)$, momentum $\mu$, weight decay $w_d$, $\eta$, $\gamma_{\min}$
    \FOR{$e = 1:N$} 
        \STATE Samples $\mathbb{B}_t = \left\{(x^1_t,y^1_t),\cdots,(x^b_t,y^b_t)\right\}$  
        \STATE Compute gradient $\nabla \mathcal{L}(w^k_t) = \frac{1}{\vert \mathbb{B}_t \vert} \sum_{i=1}^b \nabla \ell (x^i_t, y^i_t | w^k_t).$
        \STATE \colorbox{label2}{Update $\phi_(t) = \frac{1}{\alpha + e^{\psi_t}} + \gamma_{\min}$ where $\psi_t = \lambda(t - d_{\rm e})$} \\ 
        \colorbox{label2}{as mentioned in (\ref{eq:time-vary}) and $ \gamma_{\min}$ defined in (\ref{eq:phi_bound}).}
        \STATE Compute layer-wise LR $\gamma_t^k = \eta\times\phi_(t)\times\frac{\|w_t^k\|}{\|\nabla \mathcal{L}(w^k_t) + w_d\|}$
        \STATE Compute momentum $m_{t+1}^k = w_t^k - \gamma_t^k\nabla \mathcal{L}(w^k_t)$
        \STATE Adjust model weight $w_{t+1}^k = m_{t+1}^k + \mu\left(m_{t+1}^k - m_t^k\right).$
    \ENDFOR
    \end{algorithmic}
\end{algorithm}
\textbf{2) Learning Rate Decay.} To avoid instability of training due to the high initial learning rate, we use a parameter $d_{\rm e}$ specifying the number of delayed epochs as inspired by the warm-up strategy. After $d_{\rm e}$ epochs, the base learning rate is annealed via Equation (\ref{eq:time-vary}), which is the time-varying component used to tackle. According to the mathematical discussion of the LARS principle (Section \ref{sec:lars-principle}) and the aforementioned characteristic of LARS, LAMB with and without a warm-up strategy (Section \ref{sec:lars-discussion}, \ref{sec:understand-empirical-results}), the LNR $\frac{\|w\|}{\|\nabla \mathcal{L}(w)\|}$ tends to be exploding as the model is stuck at local sharp minima, then $\|\nabla \mathcal{L}(w_t)\|$ decay faster than $\|w_t\|$. 

\begin{figure}[!ht]
     \centering
     \begin{subfigure}[b]{0.21\textwidth}
         \centering
         \includegraphics[width=\textwidth]{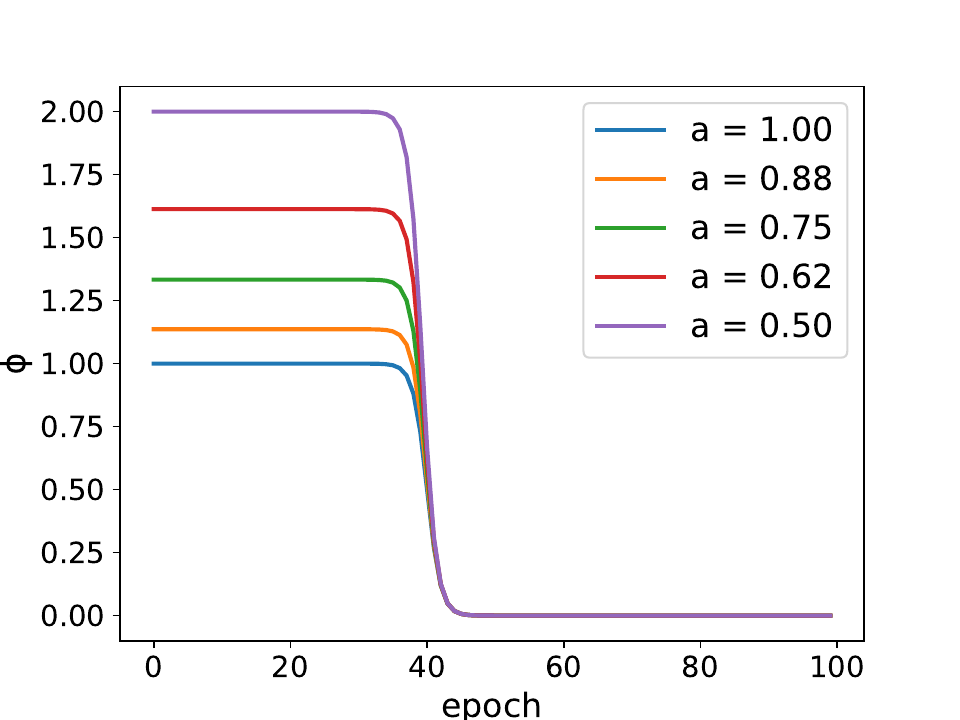}
         \caption{$\lambda = 10^{-2}$}
         \label{fig:decay-0.01}
     \end{subfigure}
     \begin{subfigure}[b]{0.21\textwidth}
         \centering
         \includegraphics[width=\textwidth]{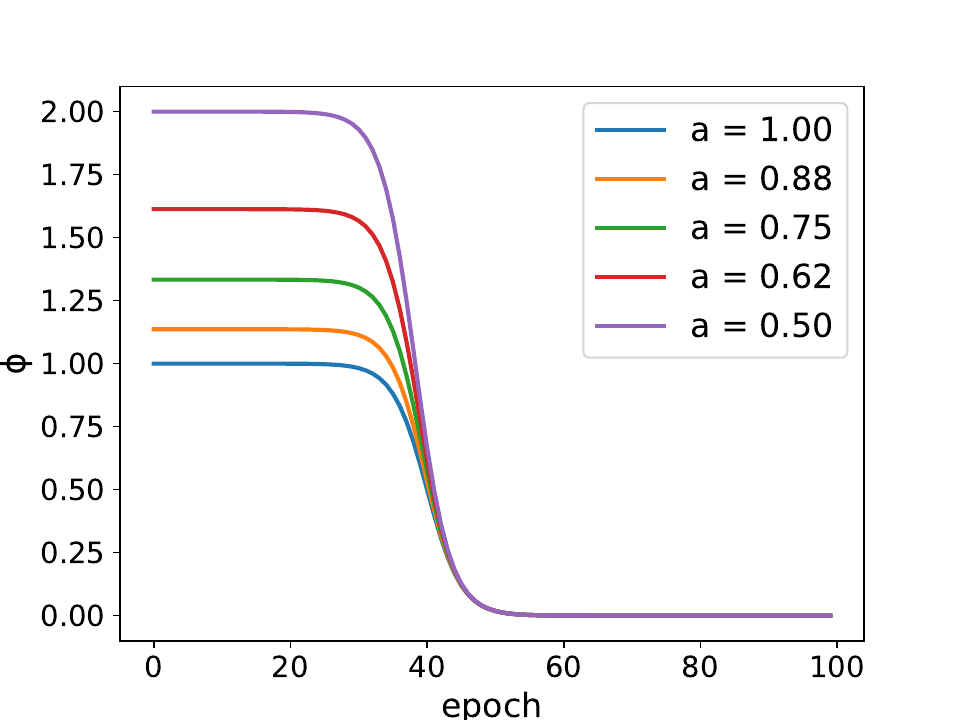}
         \caption{$\lambda = 10^{-3}$}
         \label{fig:decay-0.001}
     \end{subfigure}
     \begin{subfigure}[b]{0.21\textwidth}
         \centering
         \includegraphics[width=\textwidth]{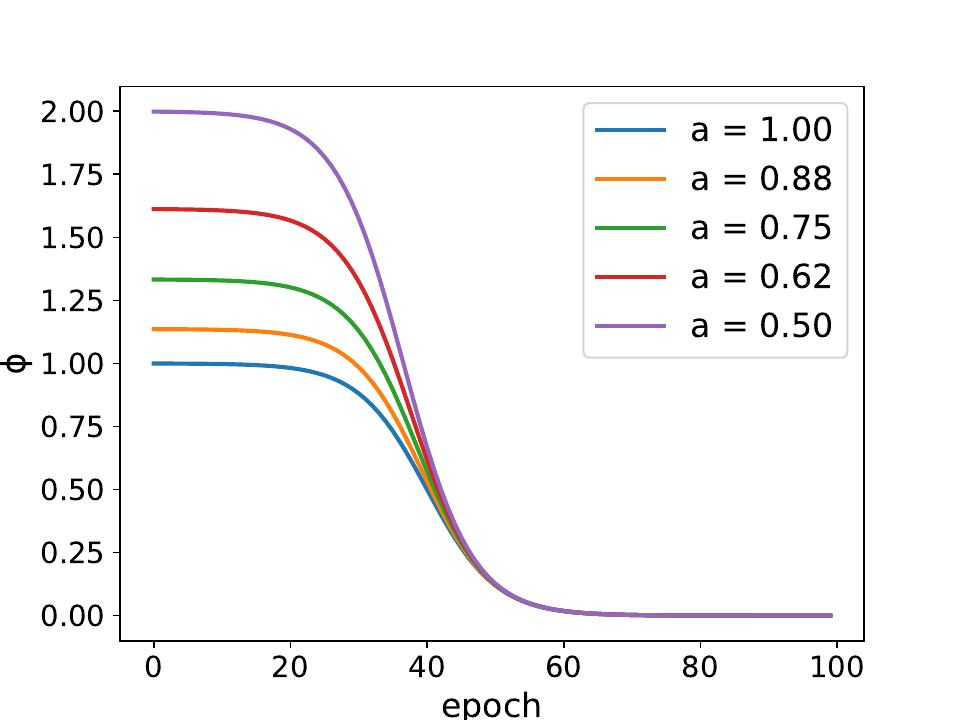}
         \caption{$\lambda = 10^{-4}$}
         \label{fig:decay-0.0001}
     \end{subfigure}
     \begin{subfigure}[b]{0.21\textwidth}
         \centering
         \includegraphics[width=\textwidth]{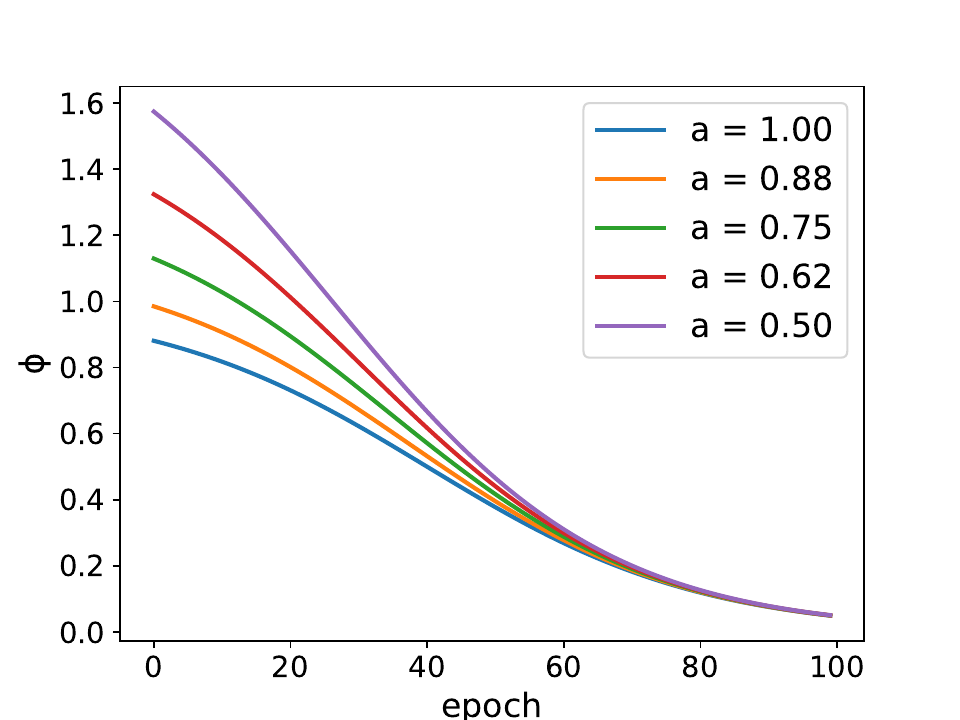}
         \caption{$\lambda = 10^{-6}$}
         \label{fig:decay-0.000001}
     \end{subfigure}
     \caption{The decay plot of TVLARS algorithm under different settings.}
     \label{fig:decay-plot}
     \vspace{-0.5cm}
\end{figure}

\textbf{3) Configurable Decay Rate.} The proposed time-varying component $\phi_(t)$ is constructed based on the sigmoid function whose curve is smooth to keep the model away from unstable learning (refers to Figure \ref{fig:decay-plot}). $\lambda$ is the soft-temperature factor, which controls the steepness of the time-variant component. Specifically, as $\lambda$ is large, the steepness is significant, and the time-variant component $\phi_(t)$ reduces faster. Therefore, by changing the soft-temperature factor $\lambda$, we can adjust the transition duration from gradient exploration to stable learning (i.e., we can achieve a stable learning phase faster as $\lambda$ is larger, and otherwise). 

\textbf{4) Alignment with LARS.} When the learning process is at the latter phase, it is essential for the TVLARS behavior to align with that of LARS to inherit LARS's robustness (refer to Figure \ref{fig:decay-plot}). We introduce two parameters $\alpha$ and $\gamma_{\min}$ used to control the bound for time-varying component $\phi_(t)$. For any $\alpha$, $\gamma_{\min}$ $\in\mathbb{R}$, the lower and upper bounds for $\phi_(t)$ is shown in Equation (\ref{eq:phi_bound}).
\begin{equation}\label{eq:phi_bound}
     \gamma_{\min} \leq \phi_(t) \leq \frac{1}{\alpha + \exp\{-\lambda d_{\rm e}\}}
\end{equation}
\begin{proof}
    We then analyzed its derivative (refers to Equation (\ref{eq:apx-bound})) to gain deeper insights into how it can affect the gradient scaling ratio. 
    
    \begin{align}\label{eq:apx-bound}
        \frac{\partial \phi(t)}{\partial t} = \frac{-\lambda\exp\{\lambda t - \lambda d_{\rm e}\}}{\left(\alpha + \exp\{\lambda t - \lambda d_{\rm e}\}\right)^2} \leq 0 \quad &\textrm{w.r.t.} 
        \\
        \begin{cases}
            \left(\alpha + \exp\{\lambda t - \lambda d_{\rm e}\}\right)^2 \geq 0 \\
            \lambda\exp\{\lambda t - \lambda d_{\rm e}\} \geq 0
        \end{cases}
    \end{align}
    
    Thus function $\phi(t)$ is a decreasing function for any $t \in \left[0, T\right)$. Therefore, the minimum value of the above function at $T \rightarrow \infty$ is $\gamma_{\min}$ as follows:
    \begin{equation}\label{eq:apx-bound-min}
        \min\{\phi(t)\} = \underset{t \rightarrow \infty}{\lim} \frac{1}{\alpha + \exp\{\lambda(t - d_{\rm e})\}} + \gamma_{\min} = \gamma_{\min}
    \end{equation}
    
    While the maximum value at $t = 0$ is as follows:
    \begin{equation}
        \max\{\phi(t)\} = \phi(t = 0) = \frac{1}{\alpha + \exp\{-\lambda d_{\rm e}\}}
    \end{equation}
    Hence the time-varying component has the following bounds:
    \begin{equation}
        \gamma_{\min} \leq \phi(t) \leq \frac{1}{\alpha + \exp\{-\lambda d_{\rm e}\}}
    \end{equation}
\end{proof}
This boundary ensures that the $\gamma^k_t$ does not explode during training. Otherwise, to guarantee that all experiments are conducted fairly, $\alpha$ is set to 1, which means there is no increment in the initial LR, and the minimum value of the LR is also set to $\gamma_{\min}$.
\begin{figure*}[!ht]
     \centering
     \begin{subfigure}[b]{\vizsizelambda\textwidth}
         \centering
         \includegraphics[width=\textwidth]{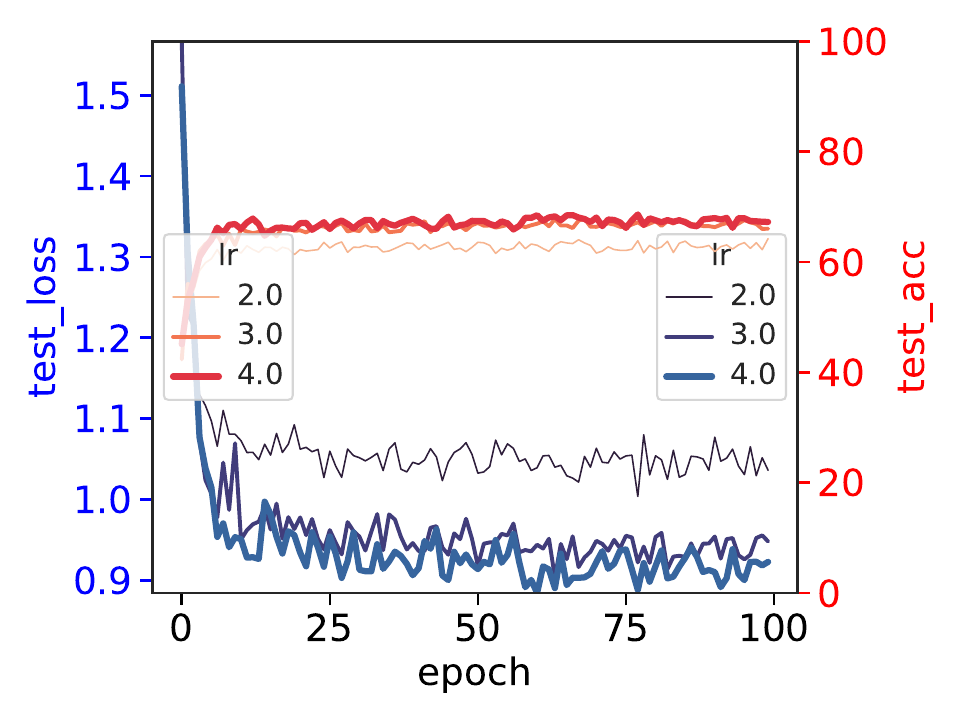}
         \caption{$\lambda = 10^{-2}$}
         \label{fig:abs-test-tvlars-0.01-1024}
     \end{subfigure}
     \begin{subfigure}[b]{\vizsizelambda\textwidth}
         \centering
         \includegraphics[width=\textwidth]{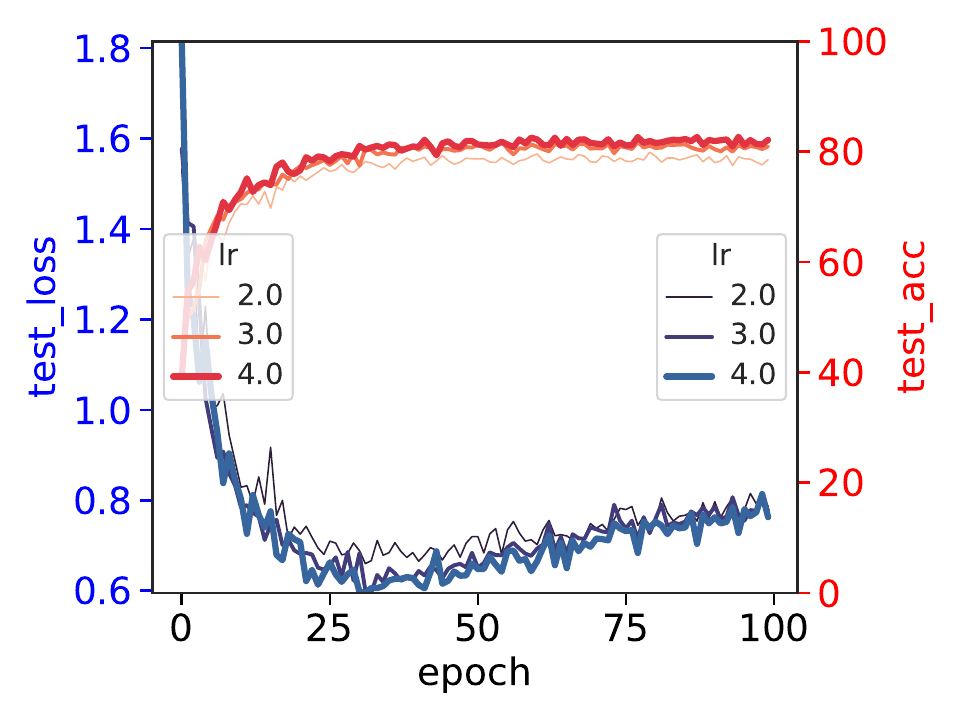}
         \caption{$\lambda = 10^{-3}$}
         \label{fig:abs-test-tvlars-0.001-1024}
     \end{subfigure}
     \begin{subfigure}[b]{\vizsizelambda\textwidth}
         \centering
         \includegraphics[width=\textwidth]{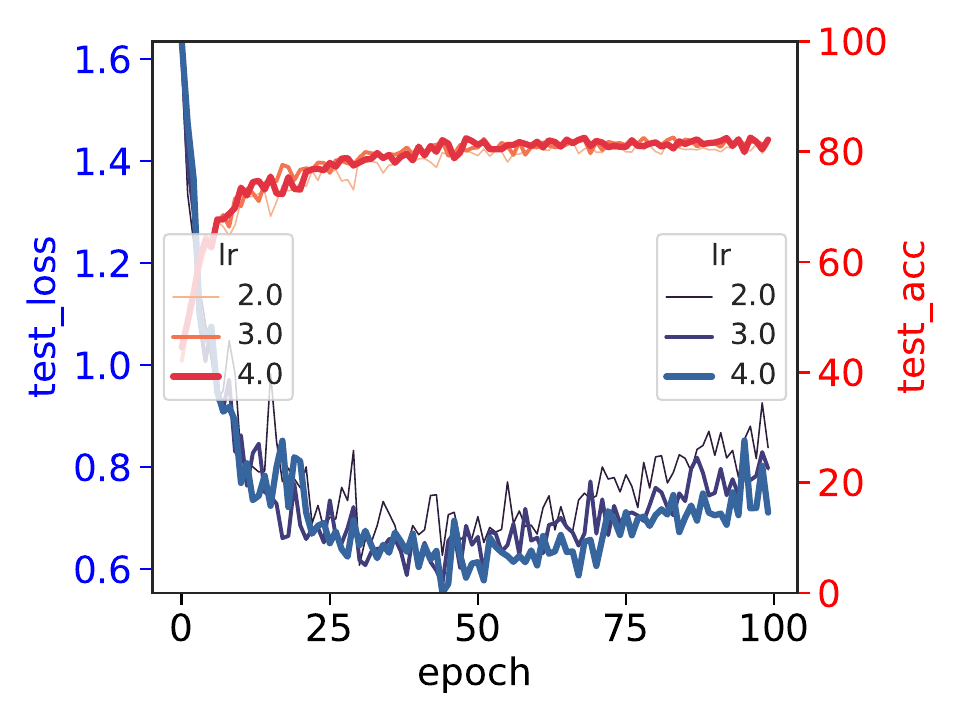}
         \caption{$\lambda = 10^{-5}$}
         \label{fig:abs-test-tvlars-0.0001-1024}
     \end{subfigure}
     \begin{subfigure}[b]{\vizsizelambda\textwidth}
         \centering
         \includegraphics[width=\textwidth]{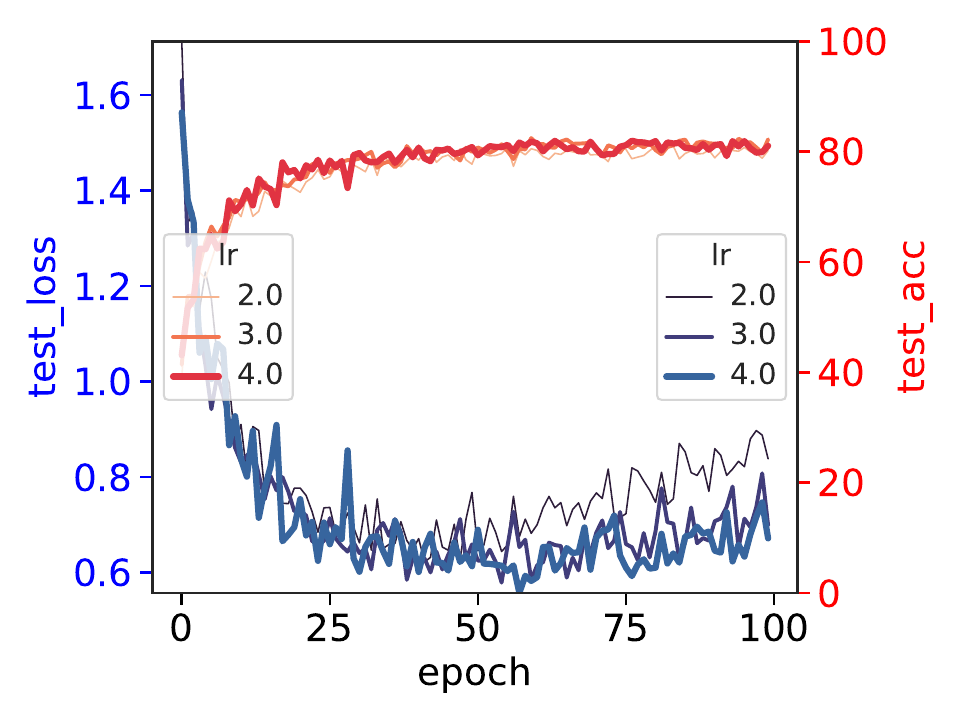}
         \caption{$\lambda = 10^{-6}$}
         \label{fig:abs-test-tvlars-0.000001-1024}
     \end{subfigure}
     \begin{subfigure}[b]{\vizsizelambda\textwidth}
         \centering
         \includegraphics[width=\textwidth]{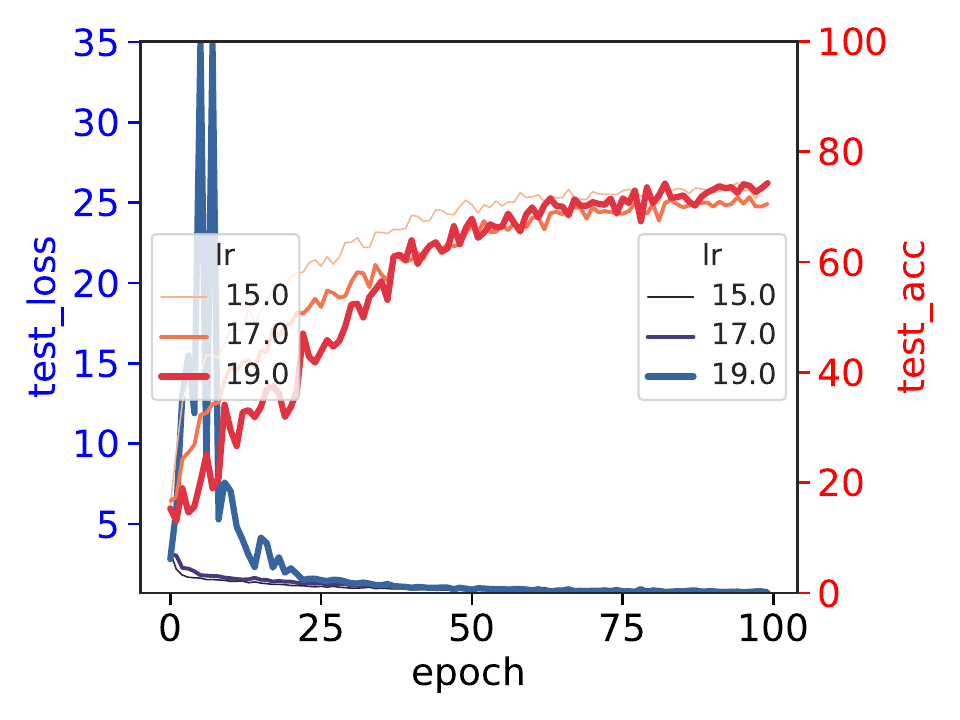}
         \caption{$\lambda = 10^{-2}$}
         \label{fig:abs-test-tvlars-0.01-16384}
     \end{subfigure}
     \begin{subfigure}[b]{\vizsizelambda\textwidth}
         \centering
         \includegraphics[width=\textwidth]{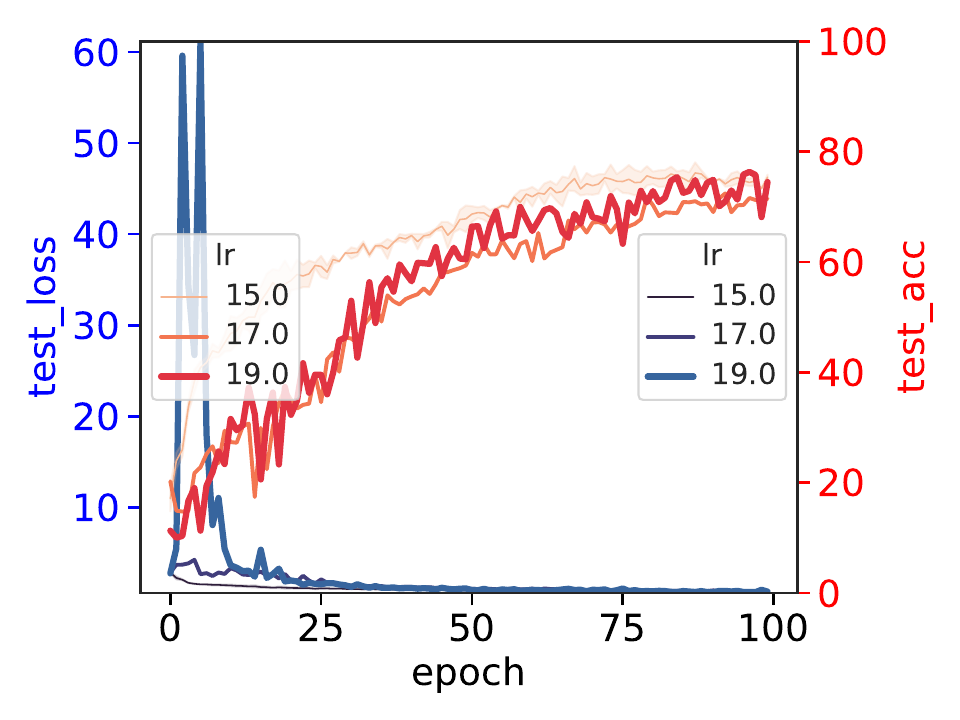}
         \caption{$\lambda = 10^{-3}$}
         \label{fig:abs-test-tvlars-0.001-16384}
     \end{subfigure}
     \begin{subfigure}[b]{\vizsizelambda\textwidth}
         \centering
         \includegraphics[width=\textwidth]{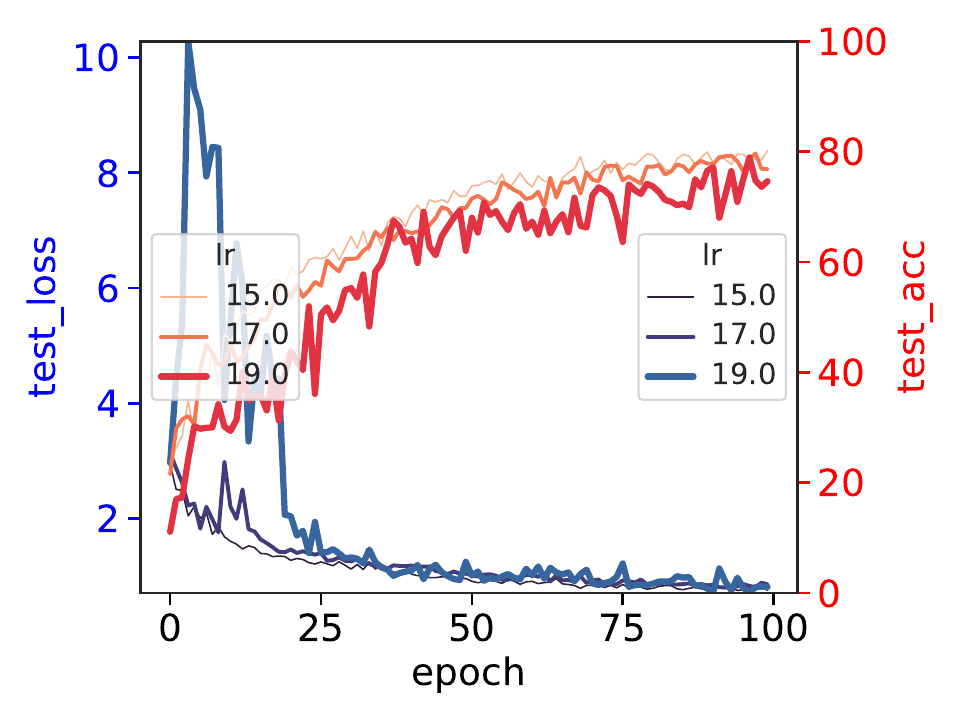}
         \caption{$\lambda = 10^{-5}$}
         \label{fig:abs-test-tvlars-0.0001-16384}
     \end{subfigure}
     \begin{subfigure}[b]{\vizsizelambda\textwidth}
         \centering
         \includegraphics[width=\textwidth]{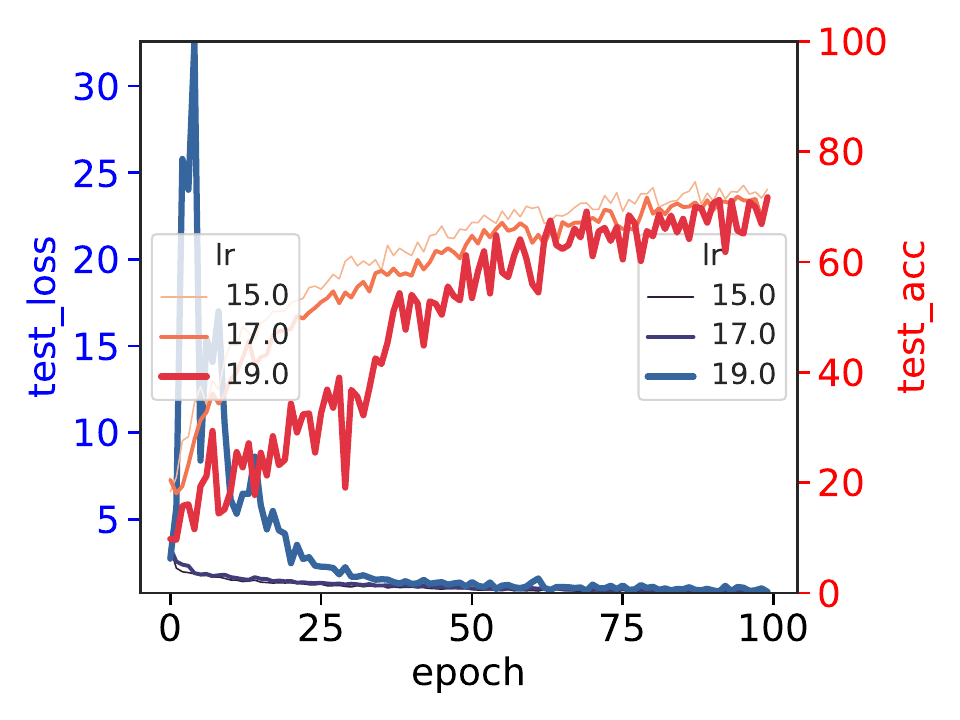}
         \caption{$\lambda = 10^{-6}$}
         \label{fig:abs-test-tvlars-0.000001-16384}
     \end{subfigure}
     \caption{Quantitative comparison in learning stability($\mathcal{B} \in \{1024, 16384\}$, which are upper and lower row, respectively).}
     \label{fig:abs-test-tvlars-decay}
\end{figure*}

\section{Experiment}\label{sec:exp-eval}
\subsection{Experimental Settings}
\textbf{Problems.} The vanilla classification (CLF) and Self Supervised Learning (SSL) are conducted and evaluated by the accuracy (\%) metric. Regarding the success of SSL, we conduct the SOTA Barlow Twins\footnote{\url{https://github.com/facebookresearch/barlowtwins}} (BT) \cite{zbontar2021barlow} to compare the performance between LARS\cite{2017-DL-LARS}, LAMB\cite{2020-DL-LAMB}, and TVLARS (ours). To be more specific, the BT SSL problem consists of two stages: SSL and CLF stage, conducted with 1000 and 100 epochs, respectively. The dimension space used in the first stage of BT is 4096 stated to be the best performance setting in \cite{zbontar2021barlow}, along with two sub Fully Connected $2048$ nodes layers integrated before the latent space layer. We also perform the CLF stage of BT with vanilla Stochastic Gradient Descent (SGD) \cite{sgd} along with the Cosine Annealing \cite{cosine} scheduler as implemented by BT authors. The main results of CLF and BT tasks are shown at \ref{exp:task}. 

\textbf{Datasets and Models.} To validate the performance of the optimizers, we consider two different data sets with distinct contexts: CIFAR10 \cite{2010-DL-Cifar10} ($32 \times 32$, 10 modalities) and Tiny ImageNet \cite{2010-DL-Imagenet} ($64 \times 64$, 200 modalities). Otherwise, the two SOTA model architectures ResNet18 and ResNet34 \cite{he2015deep} are trained separately from scratch on CIFAR10 and TinyImageNet. To make a fair comparison between optimizers, the model weight is initialized in Kaiming Uniform Distribution \cite{kaim-dis}. 

\textbf{Optimizers and Warm-up Strategy.} Specifically, we explore the characteristics of LARS and LAMB by applying them with and without a warm-up strategy, aiming to understand the LNR $\|w\|/\|\nabla\mathcal{L}(w)\|$. LARS and LAMB official source codes are implemented inside NVCaffe \cite{jia2014caffe} and Tensorflow \cite{tensorflow}. the Pytorch version of LARS used in this research is verified and referenced from Lightning Flash \footnote{\url{https://github.com/Lightning-Universe/lightning-flash}}. LAMB Pytorch code, on the other hand, verified and referenced from Pytorch Optimizer \footnote{\url{https://github.com/jettify/pytorch-optimizer}}. Besides, the warm-up strategy \cite{gotmare2018closer} contains two separate stages: linear LR scaling and LR decay. In this first stage, $\gamma_t$ becomes greater gradually by iteratively updating $\gamma_t = \gamma_\textrm{target}\times\frac{t}{T}$ for each step ($T = 20$ epochs). Then, $\gamma_t$ goes down moderately by $\gamma_t = \gamma_\textrm{target}\times q + \gamma_{min}\times (1 - q)$ where $q = \frac{1}{2}\times(1+\cos\frac{\pi t}{T})$, which is also conducted in \cite{zbontar2021barlow, chen2020simple, bardes2022vicreg}. In experiments where LARS and LAMB are conducted without a warm-up strategy, a simple Polynomial Decay is applied instead. TVLARS, on the contrary, is conducted without using a LR scheduler. 

\textbf{Hyperparameters and System.} The LRs are determined using the square root scaling rule \cite{2014-DL-WeirdTrick}, which is described detailedly at \ref{exp:abs-lr}. We considered the following sets of $\gamma_\textrm{target}$: \{1, 2, 3\}, \{2, 3, 4\}, \{5, 6, 7\}, \{8, 9, 10\}, \{10, 12, 15\}, and \{15, 17, 19\}, which are associated with $\mathcal{B}$ of 512, 1024, 2048, 4096, 8192, and 16384, respectively. Otherwise, $w_d$, and $\mu$ is set to $5\times 10^{-4}$, and $0.9$, respectively. Besides, all experiments are conducted on Ubuntu 18.04 by using Pytorch \cite{paszke2019pytorch} with multi Geforce 3080 GPUs settings, along with Syncing Batch Normalization, which is proven to boost the training performance \cite{2014-DL-WeirdTrick}, \cite{Yao_2021_CVPR}, \cite{Huang_2018_CVPR}.
\subsection{Classification and Self-supervised Problem}\label{exp:task}
Table \ref{tab:perform_cls} demonstrates the model performance trained with LARS, LAMB, and TVLARS. The table contains two main columns for CIFAR and Tiny ImageNet with associated tasks: CLF and BT. Besides CIFAR, Tiny ImageNet is considered to be a rigorously challenging dataset ($100,000$ images of $200$ classes), which is usually used to evaluate the performance of LBT SSL tasks. Overall, TVLARS achieves the highest accuracies, which outperforms LARS $4\sim7\%$, $3\sim4\%$, and $1\sim2\%$ in each pair of $\gamma_\textrm{target}$ and $\mathcal{B}$. LARS and LAMB, besides, are immobilized by the poor sharp minima indicated by the LNR $\|w\|/\|\nabla\mathcal{L}(w)\| \rightarrow \infty$ as $\|\nabla\mathcal{L}(w)\| \rightarrow 0$ (refers to Figure \ref{fig:bigwhy1}), which although creates a high adaptive LR to escape the trapped minima, $\nabla\mathcal{L}(w)/\|\nabla\mathcal{L}(w)\|$ only influences the percentage update to the layer-wise model parameters hence cannot tackle the problem of sharp minima thoroughly (more analysis at Section \ref{sec:lars-mys}). This phenomenon is caused by the warm-up strategy partly making LARS and LAMB converge slowly and be stuck at sharp minima. TVLARS ($\lambda = 10^{-3}$) reach the optimum after $20$ epochs, compared to $60\sim80$ epochs from LARS and LAMB ($\gamma_\textrm{target} = 19, \mathcal{B} = 16K$).

\subsection{Ablation test}\label{exp:abs}
\begin{figure}[!h]
     \centering
     \begin{subfigure}[b]{\vizsizeB\textwidth}
         \centering
         \includegraphics[width=\textwidth]{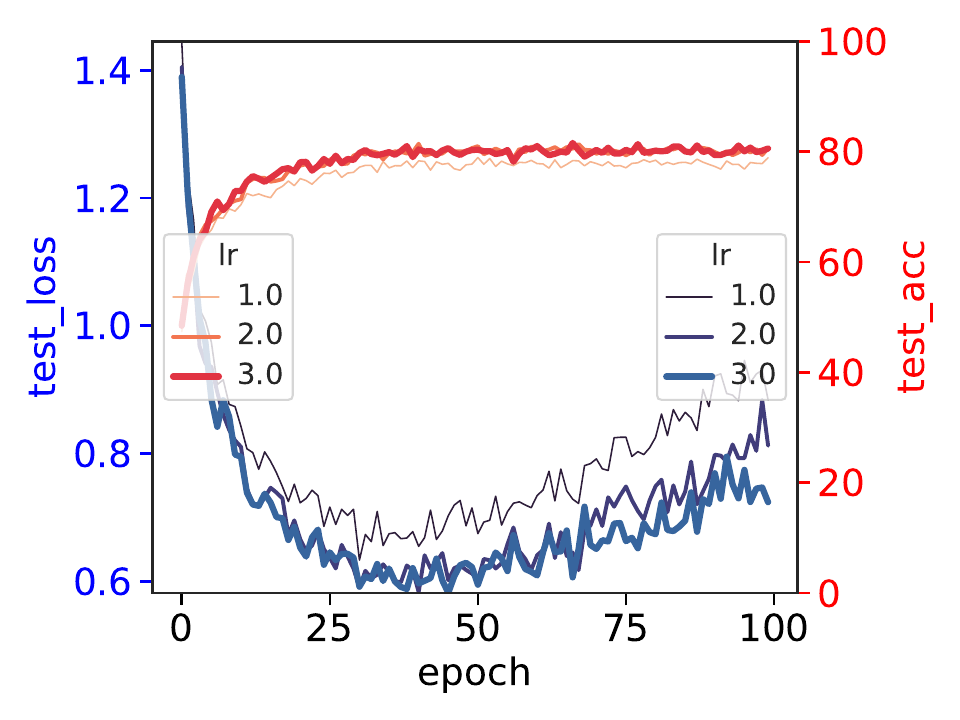}
         \caption{$\mathcal{B} = 512$}
         \label{fig:abs-testlr-tvlars-0.0001-512}
     \end{subfigure}
     \begin{subfigure}[b]{\vizsizeB\textwidth}
         \centering
         \includegraphics[width=\textwidth]{image-lib/tvlars_perform/cifar10_resnet18_xavier_uniform/tvlars_0.0001/1024/test_loss_acc.pdf}
         \caption{$\mathcal{B} = 1024$}
         \label{fig:abs-testlr-tvlars-0.0001-1024}
     \end{subfigure}
     \begin{subfigure}[b]{\vizsizeB\textwidth}
         \centering
         \includegraphics[width=\textwidth]{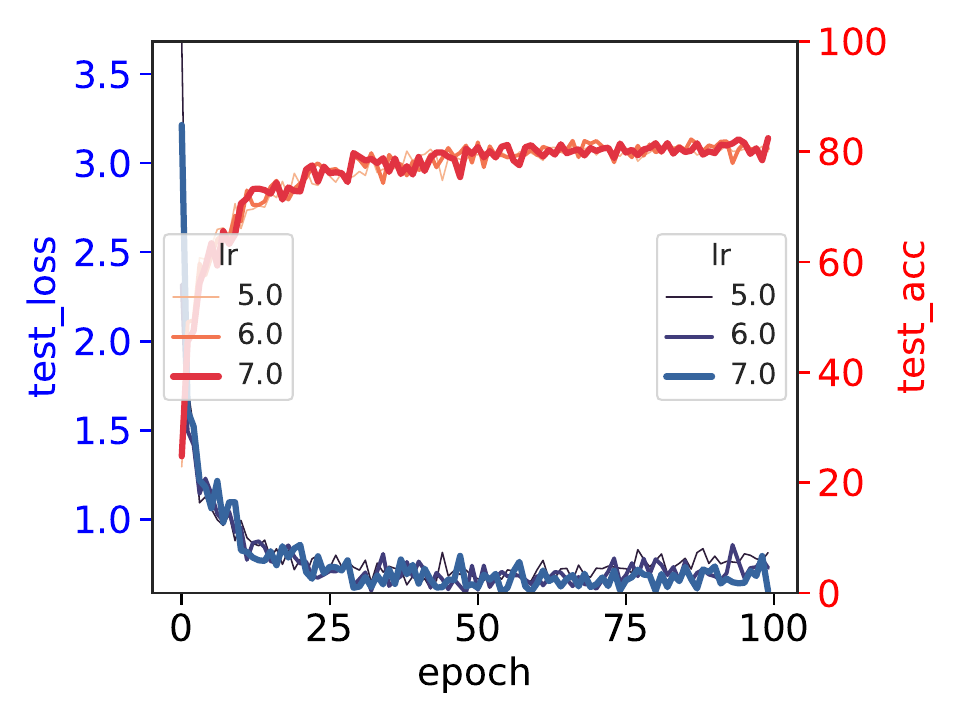}
         \caption{$\mathcal{B} = 2048$}
         \label{fig:abs-testlr-tvlars-0.0001-2048}
     \end{subfigure}
     \begin{subfigure}[b]{\vizsizeB\textwidth}
         \centering
         \includegraphics[width=\textwidth]{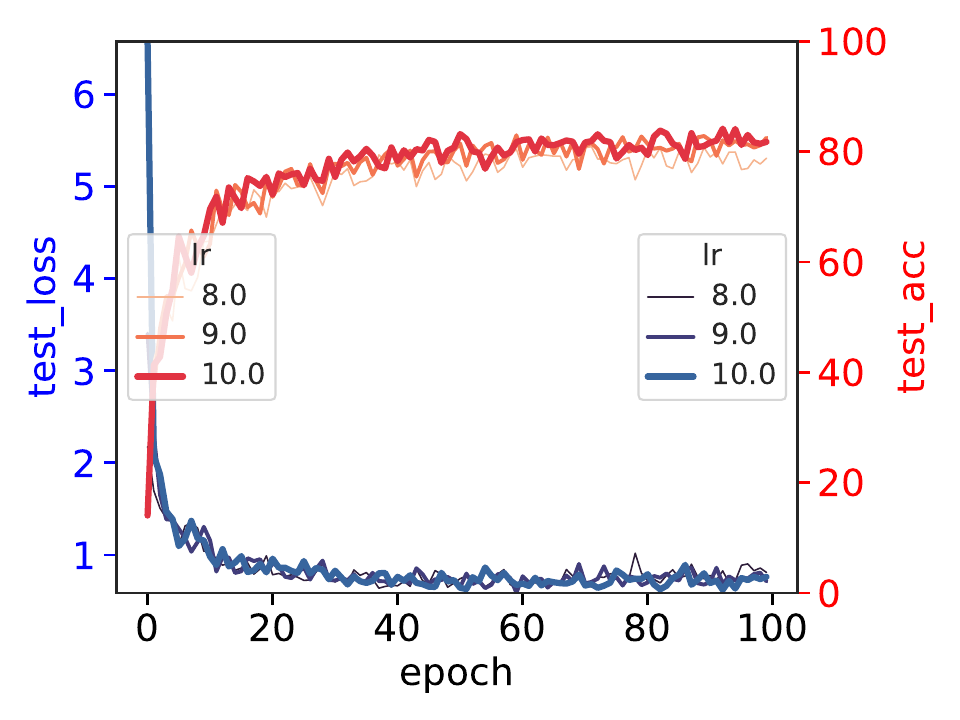}
         \caption{$\mathcal{B} = 4096$}
         \label{fig:abs-testlr-tvlars-0.0001-4096}
     \end{subfigure}
     \begin{subfigure}[b]{\vizsizeB\textwidth}
         \centering
         \includegraphics[width=\textwidth]{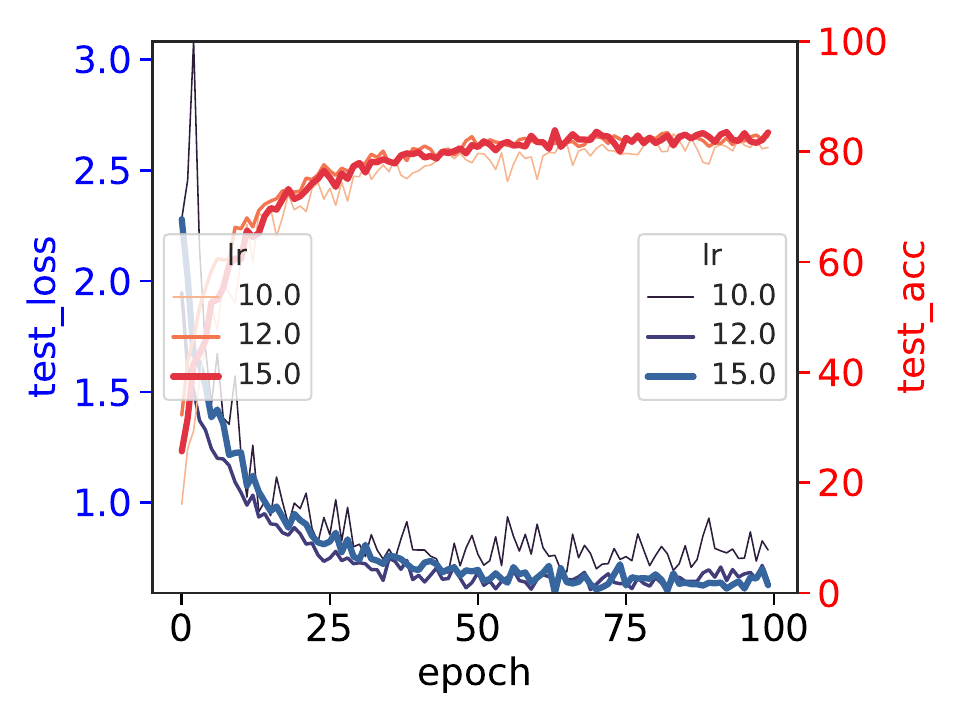}
         \caption{$\mathcal{B} = 8192$}
         \label{fig:abs-testlr-tvlars-0.0001-8192}
     \end{subfigure}
     \begin{subfigure}[b]{\vizsizeB\textwidth}
         \centering
         \includegraphics[width=\textwidth]{image-lib/tvlars_perform/cifar10_resnet18_xavier_uniform/tvlars_0.0001/16384/test_loss_acc.pdf}
         \caption{$\mathcal{B} = 16384$}
         \label{fig:abs-testlr-tvlars-0.0001-16384}
     \end{subfigure}
     \caption{Quantitative analysis of $\gamma_{\rm target}$ ($\lambda = 0.0001$)} 
     \label{fig:abs-testlr-tvlars}
\end{figure}
\subsubsection{Decay coefficients}\label{exp:abs-decay}
Decay coefficient $\lambda$ is a simple regularized parameter, used to anneal the LR to enhance the model performance. Figure \ref{fig:abs-test-tvlars-decay} demonstrates the experiments' result ($\mathcal{B} \in \{1024, 16384\}$) conducted with values of $\lambda$. Otherwise, we set $\alpha = 1$, so that the $\gamma_\textrm{target}$ for all experiments are the same. Besides, $d_{\rm e}$, the number of delay epochs is set to $10$ and $\gamma_{\rm min}$ is set to $\frac{\mathcal{B}}{\mathcal{B}_{\rm base}} \times 0.001$ for both TVLARS and LARS experiment. 


In 1K batch-sized experiments, there is a large generalization gap among $\gamma_\textrm{target}$ for $\lambda \in \{0.01, 0.005\}$. Smaller $\lambda$, otherwise, enhance the model accuracy by leaving $\gamma_\textrm{target}$ to stay nearly unchanged longer, which boosts the ability to explore loss landscape and avoid sharp minima. As a result, the model achieve higher accuracy: $\sim 84\%$ ($\lambda = 10^{-5}$), compared to result stated in Table \ref{tab:perform_cls}. In contrast, the model performs better with larger values of $\lambda$ (i.e. 0.01, 0.005, and 0.001) in 16K batch-sized experiments. Owing to high initial $\gamma_\textrm{target}$, it is easier for the model to escape the sharp minima which do not only converge within $20$ epochs (four times compared to LARS) but also to a low loss value ($\sim 2$), compared to just under $20$ ($\lambda = 10^{-6}$). This problem is owing to the elevated $\gamma_\textrm{target}$, which makes the leaning direction fluctuate dramatically in the latter training phase (refers to Figures \ref{fig:abs-test-tvlars-0.0001-16384}, \ref{fig:abs-test-tvlars-0.000001-16384}). 
\begin{figure}[!h]
     \centering
     \begin{subfigure}[b]{\vizsizeweight\textwidth}
         \centering
         \begin{subfigure}[b]{0.49\textwidth}
             \includegraphics[width=\textwidth]{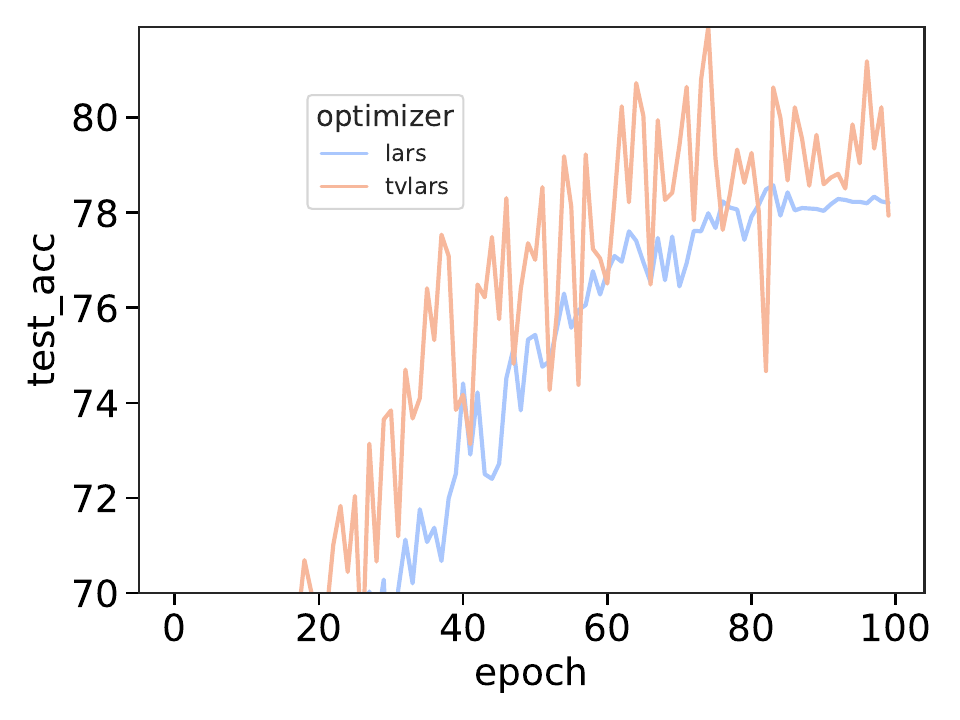}   
         \end{subfigure}
         \begin{subfigure}[b]{0.49\textwidth}
             \includegraphics[width=\textwidth]{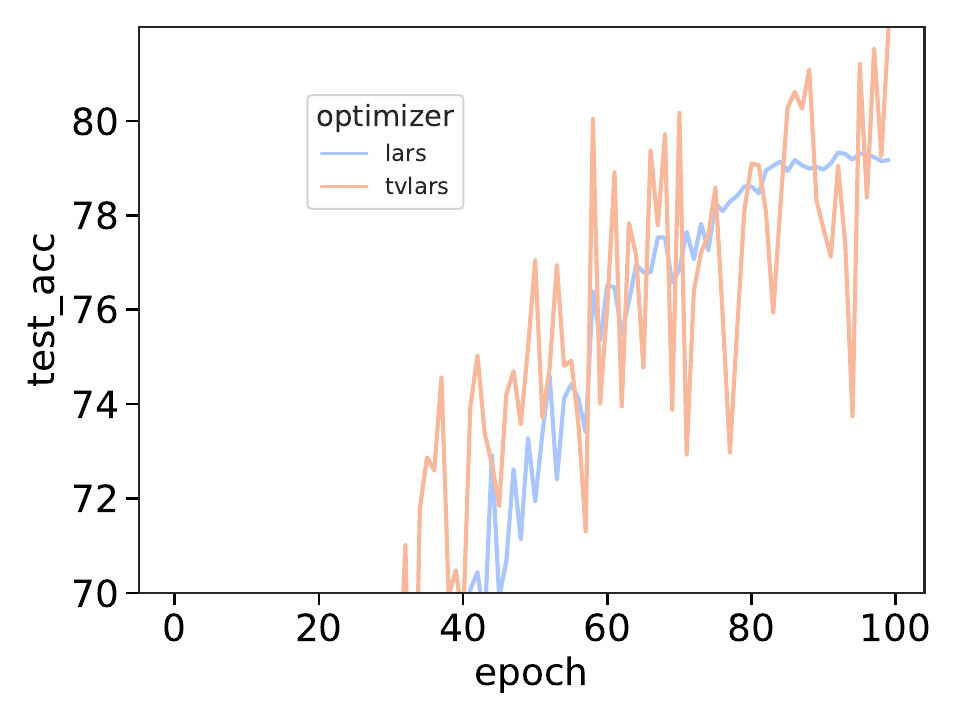}
         \end{subfigure}
     \end{subfigure}
     \begin{subfigure}[b]{\vizsizeweight\textwidth}
         \centering
         \begin{subfigure}[b]{0.49\textwidth}
             \includegraphics[width=\textwidth]{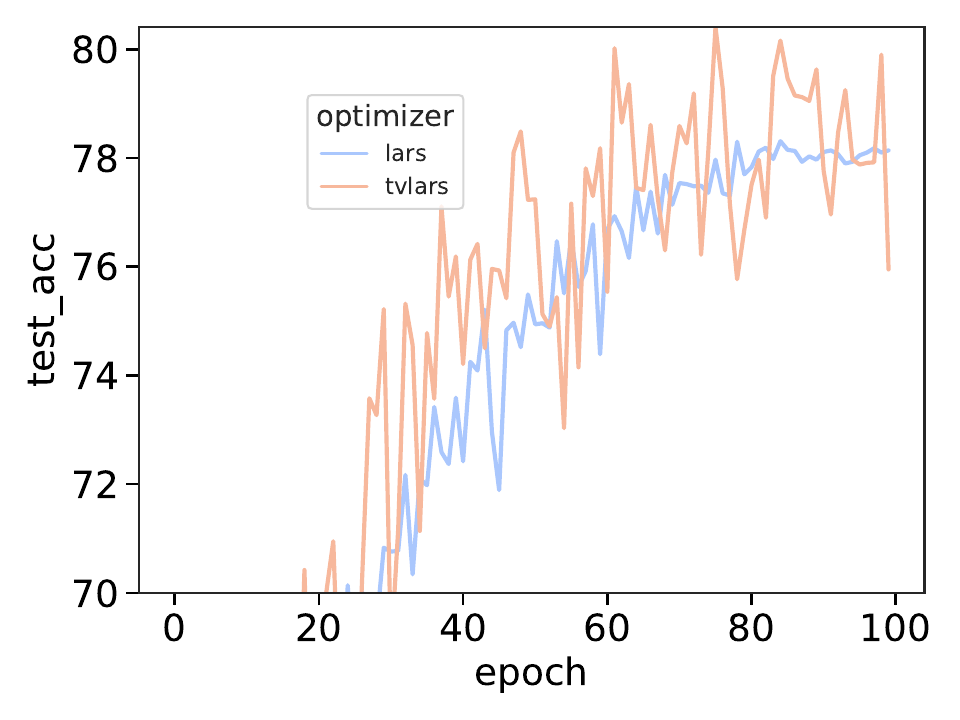}   
         \end{subfigure}
         \begin{subfigure}[b]{0.49\textwidth}
             \includegraphics[width=\textwidth]{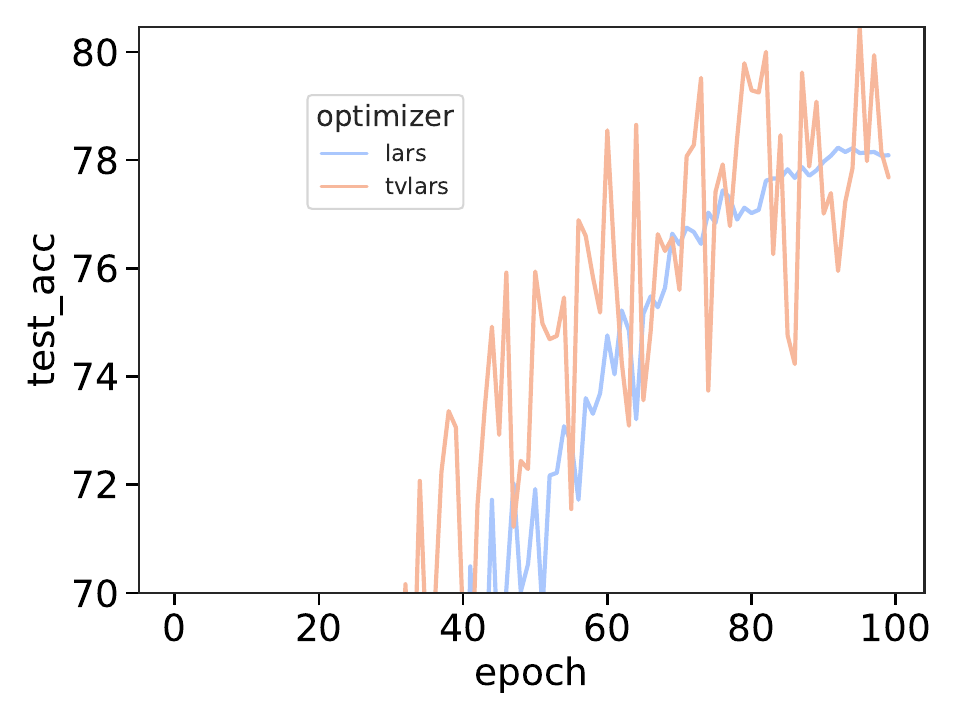}
         \end{subfigure}
         \caption{Xavier Normal}
         \label{fig:abs-winit-xn}
     \end{subfigure}
     \begin{subfigure}[b]{\vizsizeweight\textwidth}
         \centering
         \begin{subfigure}[b]{0.49\textwidth}
             \includegraphics[width=\textwidth]{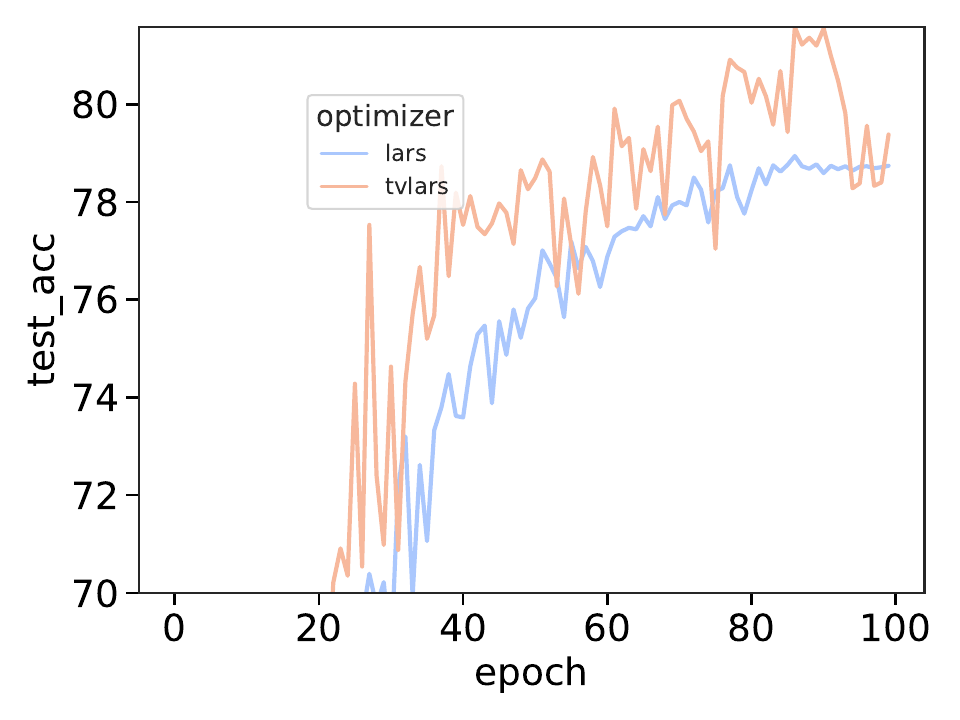}   
         \end{subfigure}
         \begin{subfigure}[b]{0.49\textwidth}
             \includegraphics[width=\textwidth]{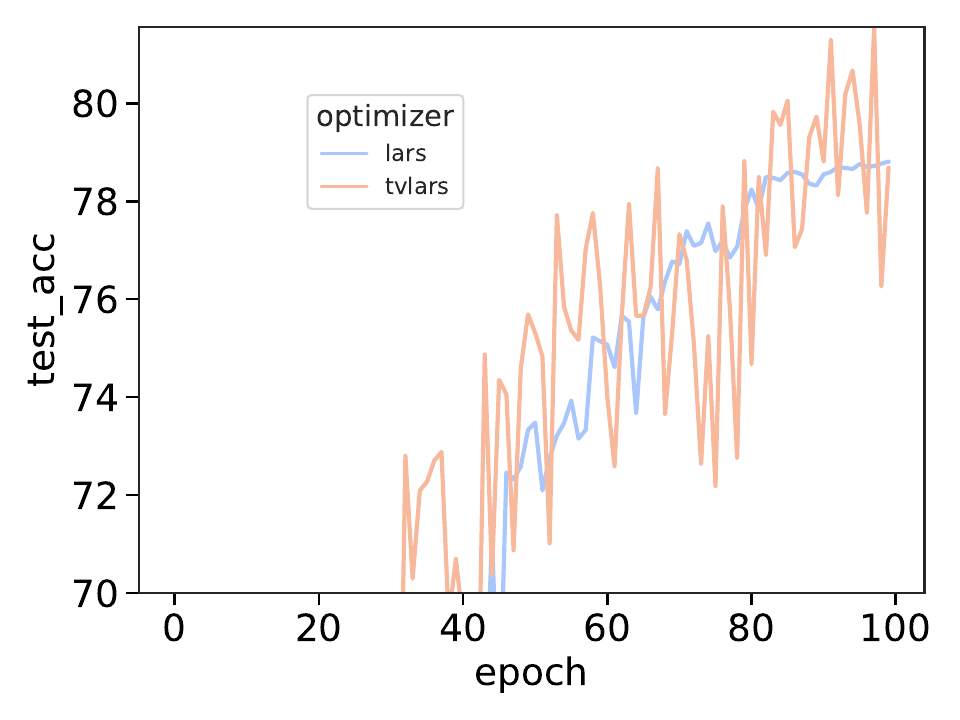}
         \end{subfigure}
     \end{subfigure}
     \begin{subfigure}[b]{\vizsizeweight\textwidth}
         \centering
         \begin{subfigure}[b]{0.49\textwidth}
             \includegraphics[width=\textwidth]{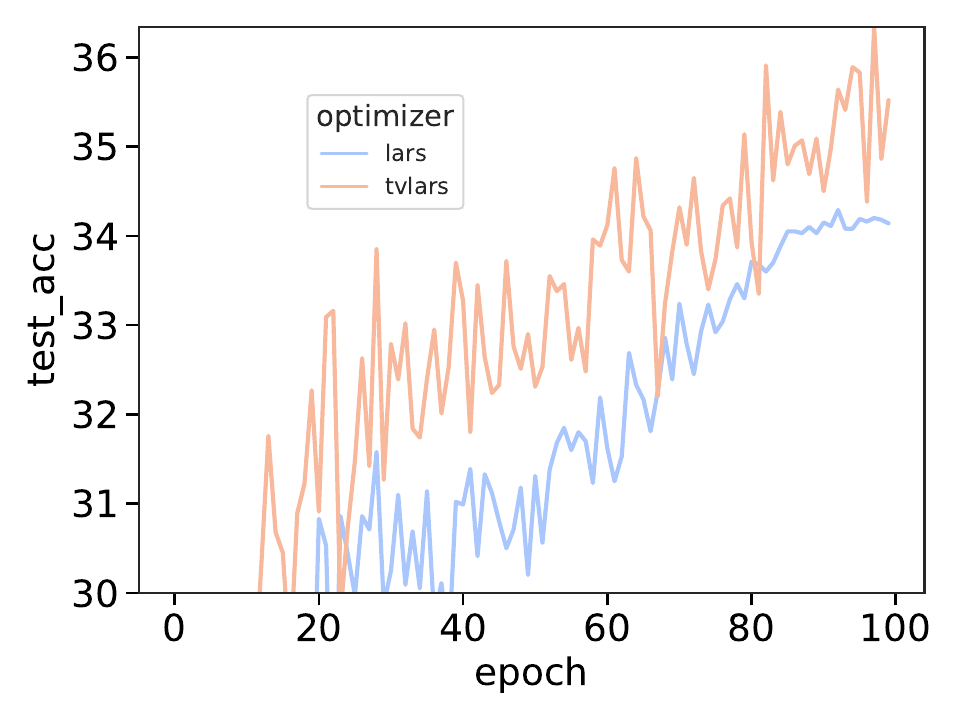}   
         \end{subfigure}
         \begin{subfigure}[b]{0.49\textwidth}
             \includegraphics[width=\textwidth]{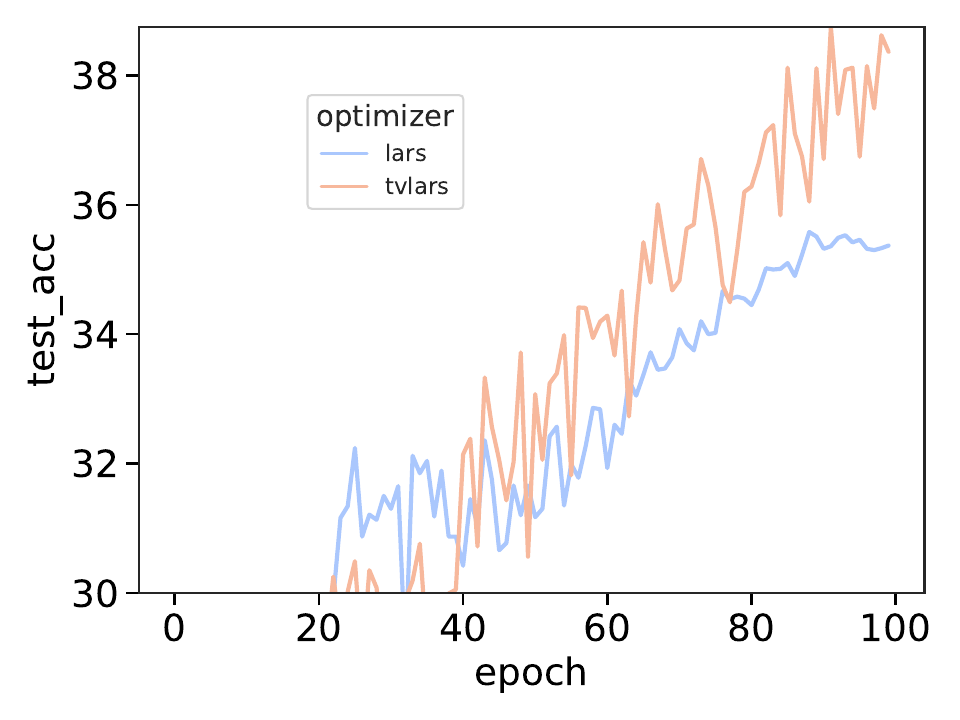}
         \end{subfigure}
         \caption{Kaiming Normal}
         \label{fig:abs-winit-kn}
     \end{subfigure}
     \begin{subfigure}[b]{\vizsizeweight\textwidth}
         \centering
         \begin{subfigure}[b]{0.49\textwidth}
             \includegraphics[width=\textwidth]{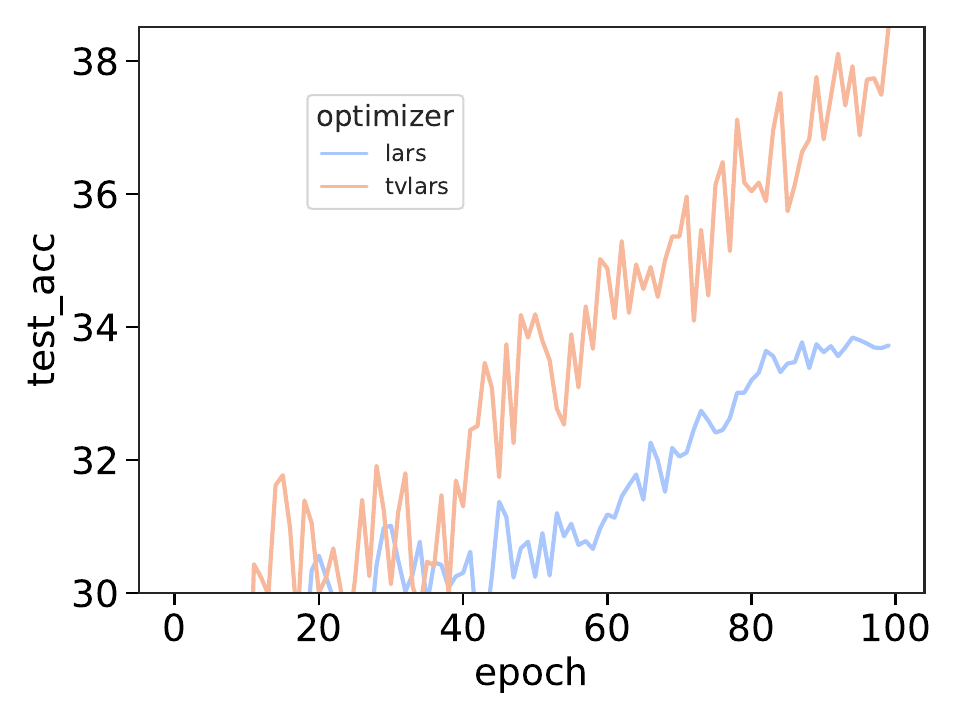}   
         \end{subfigure}
         \begin{subfigure}[b]{0.49\textwidth}
             \includegraphics[width=\textwidth]{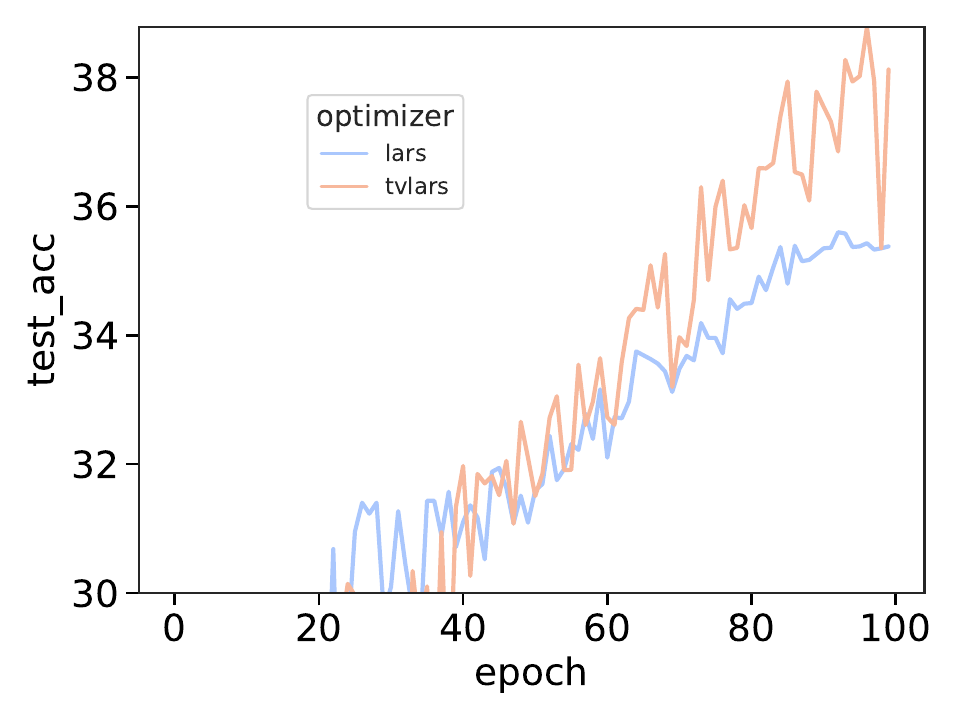}
         \end{subfigure}
     \end{subfigure}
     \begin{subfigure}[b]{\vizsizeweight\textwidth}
         \centering
         \begin{subfigure}[b]{0.49\textwidth}
             \includegraphics[width=\textwidth]{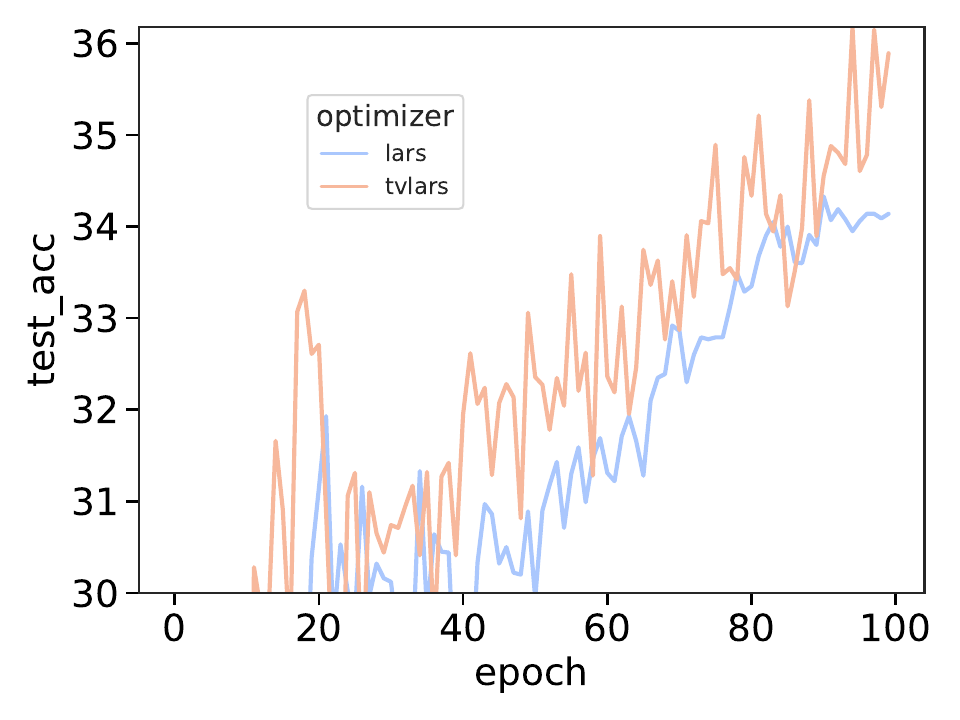}   
         \end{subfigure}
         \begin{subfigure}[b]{0.49\textwidth}
             \includegraphics[width=\textwidth]{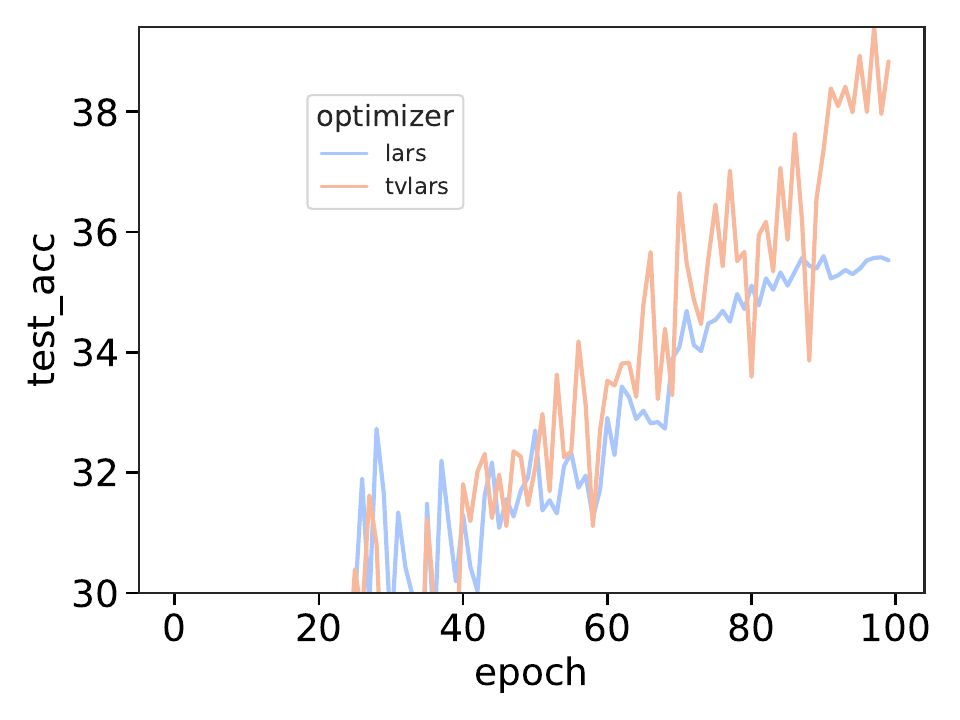}
         \end{subfigure}
         \caption{Kaiming Uniform}
         \label{fig:abs-winit-ku}
     \end{subfigure}
     \caption{Quantitative analysis of different weight initialization methods for CIFAR10 and TinyImageNet (upper and lower rows). For each method, $\mathcal{B}\in\{8192, 16384\}$ (left, right columns)}. 
     \label{fig:abs-winit}
     \vspace{-0.5cm}
\end{figure}

\subsubsection{Learning rate}\label{exp:abs-lr}
A high initial learning rate (LR), otherwise, plays a pivotal role in enhancing the model performance by sharp minimizer avoidance \cite{2018-LR-SharpMinia, 2017-ShartMinima-Generalize}. Authors of \cite{2022-LR-Decay1, 2014-DL-WeirdTrick} suggest that, when $\mathcal{B}/\mathcal{B}_{\rm base} = m$, the LR should be $\epsilon\sqrt{m}$ to keep the variance, where $\epsilon$ is the LR used with $\mathcal{B}_\textrm{base}$. However, choosing $\epsilon$ is an empirical task, hence we do not only apply the theorem from \cite{2018-LR-SharpMinia} but also conduct the experiments with LRs in a large variation to analyze how LR can affect the model performance. Figure \ref{fig:abs-testlr-tvlars} illustrates that the higher the LR is the lower the loss and the higher the accuracy the model can achieve. 

\subsubsection{Weight Initialization}\label{exp:abs-winit}
According to \cite{2017-DL-LARS}, the weight initialization is sensitive to the initial training phase. From Equation (\ref{eq:LARS}), when the value of $\gamma_t^k$ is high due to the ratio $\mathcal{B}/\mathcal{B}_{\rm base}$ (i.e. $\mathcal{B} = 16K$), the update magnitude of $\|\gamma_t^k \nabla\mathcal{L}(w_t^k)\|$ may outperform $\|w_t^k\|$ and cause divergence. Otherwise, since $w \sim \mathcal{P}(w_0)$ (weight initialization distribution), which makes $\Vert w\Vert$ varies in distinguished variation, hence the ratio LNR $\|w\|/\|\nabla\mathcal{L}\|$ may make the initial training phase performance different in each method of weight initialization. Addressing this potential phenomenon, apart from Xavier Uniform \cite{xavier-winit}, which has been shown above, we conduct the experiments using various types of weight initialization: Xavier Normal \cite{xavier-winit} and Kaiming He Uniform, Normal \cite{kaimhe-winit}. It is transparent that, the model performance results using different weight initialization methods are nearly unchanged. TVLARS, though its performance is unstable owing to its exploration ability, outperforms LARS $1\sim3\%$ in both CIFAR10 and Tiny ImageNet.

\section{Conclusion}
In this paper, we have proposed a new method, called TVLARS, for large batch training (LBT) in neural networks, which outperforms the state of the art in LBT by addressing the current shortcomings of the existing methods. In particular, we first conducted extensive experiments to gain deeper insights into layerwise-based update algorithms to understand the causes of the shortcomings of layerwise adaptive learning rates in LBTs. Based on these findings, we designed TVLARS to capitalize on the observation that LBT often encounters sharper minimizers during the initial stages. By prioritizing gradient exploration, we facilitated more efficient navigation through these initial obstacles in LBT. Simultaneously, through adjustable discounts in layerwise LRs, TVLARS combines the favorable aspects of a sequence of layerwise adaptive LRs to ensure strong convergence in LBT and overcome the issues of warm-up. With TVLARS we achieved significantly improved convergence compared to two other cutting-edge methods, LARS and LAMB, especially combined with warm-up and when dealing with extremely LB sizes (e.g., $\mathcal{B}=16384$), across Tiny ImageNet and CIFAR-10 datasets.

\bibliographystyle{ieeetr}
\bibliography{tai-2024/ref}
\end{document}